\documentclass{IEEEtaes}

\usepackage{color,array,amsthm}
\usepackage{amsmath,amsfonts,mathtools}
\usepackage{graphicx}
\usepackage{booktabs}
\usepackage[dvipsnames]{xcolor}
\usepackage{colortbl}
\usepackage{empheq}
\usepackage{cite}

\jvol{XX}
\jnum{XX}
\jmonth{XXXXX}
\paper{1234567}
\pubyear{2026}
\doiinfo{TAES.2020.Doi Number}

\newtheorem{theorem}{Theorem}
\newtheorem{proposition}{Proposition}
\newtheorem{remark}{Remark}

\newtheorem{assumption}{Assumption}
\newtheorem{problem}{Problem}

\let\b\mathbf
\let\bs\boldsymbol

\begin{document}

\title{\textit{QuadRocket}: An Aerial Robotic Testbed for Adaptive Thrust-Vector Control of Rocket-Like Vehicles}

\author{PEDRO SANTOS}
\member{Graduate Student Member, IEEE}
\affil{Instituto Superior Técnico, Universidade de Lisboa, Lisbon, Portugal} 

\author{JOEL REIS}
\affil{Faculty of Science and Technology, University of Macau,
	Taipa, Macao, China} 

\author{PAULO OLIVEIRA}
\member{Senior Member, IEEE}
\affil{Instituto Superior Técnico, Universidade de Lisboa, Lisbon, Portugal}

\author{CARLOS SILVESTRE}
\member{Senior Member, IEEE}
\affil{Faculty of Science and Technology, University of Macau,
	Taipa, Macao, China}

\receiveddate{This article has been accepted for publication in IEEE Transactions on Aerospace and Electronic Systems. This is the author's version which has not been fully edited and
	content may change prior to final publication. Citation information: DOI 10.1109/TAES.2026.3706328. \copyright 2026 IEEE. Personal use of this material is permitted. Permission from IEEE must be obtained for all other uses, in any current or future media, including reprinting/republishing this material for advertising or promotional purposes, creating new collective works, for resale or redistribution to servers or lists, or reuse of any copyrighted component of this work in other works\\%
This work was supported %
in part by the Macao Science and Technology Development Fund under Grants FDCT/0192/2023/RIA3 and FDCT/0059/2024/RIA1, %
in part by the University of Macau, under Projects MYRG2022-00205-FST, MYRG-GRG2023-00107-FST-UMDF, UMDF-TISF/2025/007/FST and SRG2024-00012-FST, %
and in part by Fundação para a Ciência e Tecnologia (FCT) through LAETA (UID/50022/2025) and LARSyS (DOI: 10.54499/LA/P/0083/2020). %
Pedro Santos holds a Ph.D. scholarship from FCT (2023.00268.BD). }

\corresp{{\itshape (Corresponding author: C. Silvestre)}.}

\authoraddress{P. Santos and P. Oliveira are with IDMEC and ISR, Instituto Superior Técnico, Universidade de Lisboa, Lisbon, Portugal (e-mail: \href{mailto:pedrodossantos31@tecnico.ulisboa.pt}{pedrodossantos}; \href{mailto:paulo.j.oliveira@tecnico.ulisboa.pt}{paulo.j.oliveira}). J. Reis and C. Silvestre are with the Department of Electrical and Computer Engineering of the Faculty of Science and Technology, University of Macau, Taipa, Macao, China (e-mail: \href{mailto:joelreis@um.edu.mo}{joelreis}; \href{mailto:csilvestre@um.edu.mo}{csilvestre}). C. Silvestre is on leave from Instituto Superior Técnico, Universidade de Lisboa, Lisboa, Portugal.}


\markboth{SANTOS ET AL.}{AN AERIAL ROBOTIC TESTBED FOR ADAPTIVE THRUST-VECTOR CONTROL}
\maketitle

\begin{abstract}This paper presents \textit{QuadRocket}, a quadrotor-based rocket prototype that provides a low-cost, low-risk platform for validating advanced thrust-vector control strategies for launch-vehicle–type systems. The prototype consists of a cylindrical main body mounted on top of a quadrotor through a universal joint, forming a flying inverted pendulum with non-negligible inertia. For control design, the coupled system is modeled as a single axisymmetric rigid body actuated by a vectored force applied along its longitudinal axis. A reduced-attitude representation on the two-sphere is adopted to explicitly exploit the vehicle’s axial symmetry and to decouple yaw from the thrust-vector direction. On this model, we derive an adaptive backstepping controller that achieves almost-global trajectory tracking in the presence of unknown constant disturbances, while a control-point transformation mitigates non-minimum-phase behavior. The quadrotor is then treated as a thrust-vector actuator, and a dynamic-surface-based attitude controller is designed to track the desired thrust-vector, accounting for actuation dynamics and avoiding explicit differentiation of virtual control signals. The complete architecture is evaluated in simulation and validated experimentally in an indoor motion-capture arena. Results demonstrate accurate trajectory tracking, effective disturbance compensation, and confirm the suitability of the \textit{QuadRocket} as a versatile testbed for thrust-vector-controlled robotic vehicles.
\end{abstract}


\section{INTRODUCTION}
T{\scshape esting} and validating control algorithms before deployment in highly complex autonomous vehicles is of utmost importance. This is especially true in the launch vehicle industry, where strict safety requirements and high operational costs impose that new technologies demonstrate a proven level of reliability \cite{Bennani2024}. At the same time, increasingly demanding mission scenarios require novel control algorithms capable of coping with a broader range of operating conditions and adapting to varying trajectory tracking requirements across different missions \cite{Simplicio2019}. Nonlinear and adaptive control strategies are natural candidates due to their intrinsic suitability for achieving global—or near-global—stability and tracking performance in the presence of unknown disturbances and parametric uncertainties \cite{Aguiar2007}. Because these algorithms are more complex and less established than classical control approaches, low-cost and low-risk test vehicles become essential tools for thoroughly verifying and validating their performance.
\begin{figure}[t]
	\centering
	\includegraphics[width=2.5in]{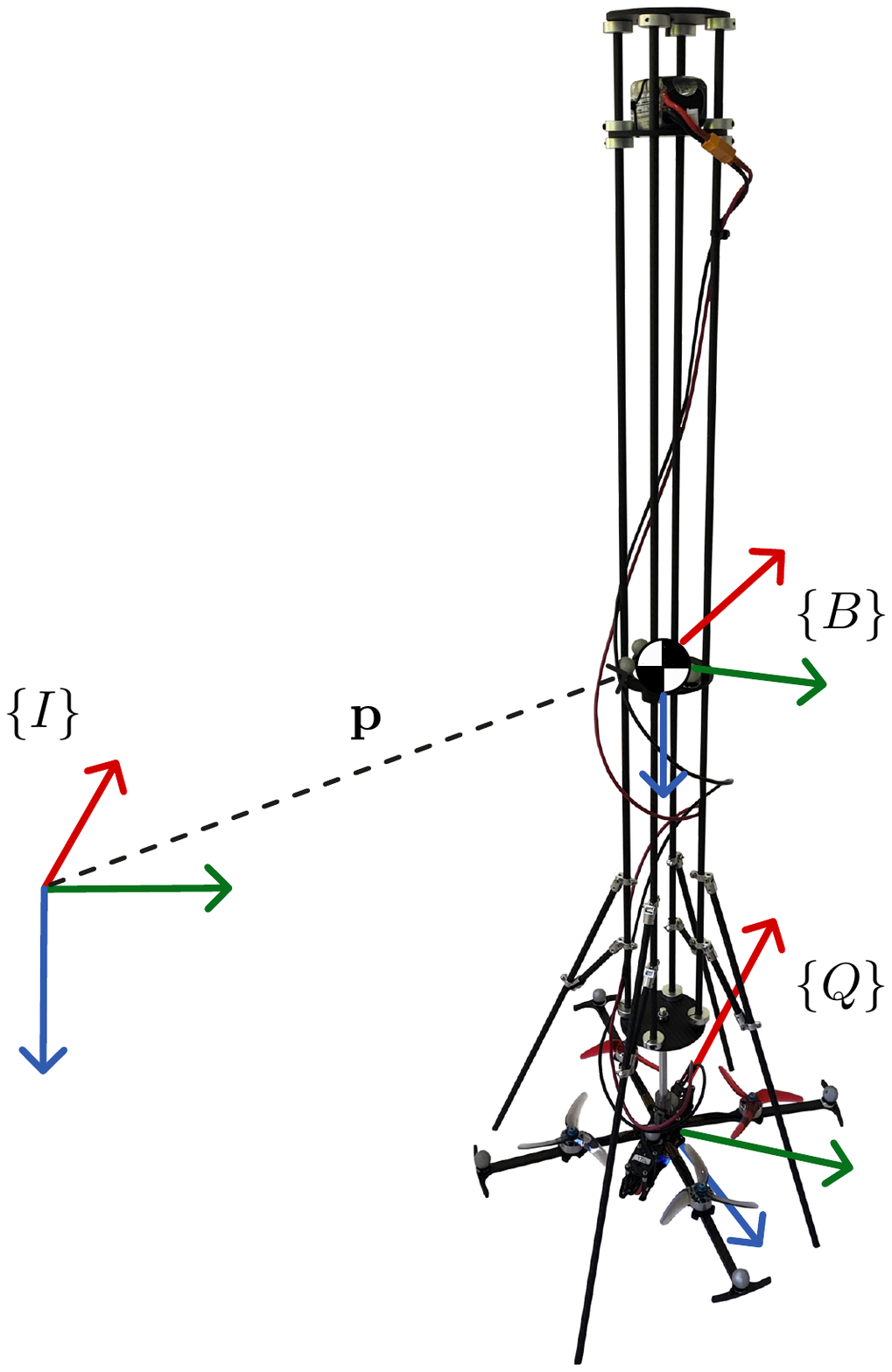}
	\caption{The \textit{QuadRocket}: representation of reference frames (red, green, and blue represent the $x$, $y$, and $z$ axes, respectively).}
	\label{fig:quadrocket}
\end{figure}

In this work, a novel platform to serve as low-cost testbed for thrust vectoring-based advanced control algorithms is presented and flight-tested. Leveraging the wide availability and cost-effective nature of quadrotor UAVs (Unmanned Aerial Vehicles) \cite{Sun2022}, together with their agility and extensive control-oriented research body, a cylindrical carbon fiber structure is designed and assembled on top of a quadrotor through an universal joint. The universal joint allows for unconstrained rotational motion between both bodies in two degrees of freedom, resulting in a flying inverted pendulum configuration with non-negligible inertia. This system is physically analogous to a rocket controlled via thrust vectoring in the sense that the net thrust force is applied along a direction that is not fixed to the main body axis, generating coupled translational and rotational dynamics through an off-center actuation point. As in rocket Thrust-Vector Control (TVC) systems, control authority is exercised through the regulation of the thrust-vector direction and magnitude, leading to similar control objectives and stability challenges. Besides its primary goal of serving as a low-cost and low-risk testbed for advanced thrust-vector control of rocket-like vehicles, we envision two additional applications which naturally emerge from the proposed aerial vehicle: i) autonomous cargo transportation, in which the cylindrical structure could carry a given payload; and ii) autonomous inspection of civil infrastructures, namely ceilings \cite{Kocer2019}, where the structure mounted on the quadrotor enables agile close-proximity operations.

 For control design purposes, we follow the rocket analogy and take the two interconnected systems as a single rigid body  with the inertial orientation of the carbon structure, actuated by a vectored force, which is given by the relative orientation and thrust magnitude of the quadrotor. In other words, the quadrotor is seen as a thrust-vector actuator to the main cylindrical body. Consequently, our control strategy focuses on two main tasks: i) to design a trajectory tracking controller for a generic axisymmetric rigid body actuated by a vectored force applied at a point on its longitudinal axis, below the center of mass; and ii) to design a control law for the angular velocity command of the quadrotor such that it tracks the desired thrust-vector provided by the flight controller.
 
  The axisymmetric rigid body actuated by a vectored force is an underactuated system with additional nonholonomic constraints arising from the coupling between force and torque inputs. This coupling can also lead to undesired non-minimum phase behavior \cite{Hauser1992}. The literature is scarce when it comes to integrated nonlinear trajectory tracking solutions for thrust-vector-controlled rockets, as position and attitude control loops are often designed separately, relying on linear approximations, and without formal stability guarantees for the complete cascaded system \cite{Wie2008}. Nevertheless, integrated approaches are standard for other underactuated autonomous vehicles, some of which sharing fundamental actuation mechanics with our platform. These include Lyapunov-based methods, most predominantly nonlinear backstepping, for thrust-vectored-controlled surface crafts \cite{Reis2023,Cabecinhas2018} and VTOL (Vertical Take-Off and Landing) ducted fan UAVs \cite{Pflimlin2004}, geometric control of quadrotor UAVs \cite{Lee2010,Gamagedara2019,Shi2017}, sliding mode control for missiles \cite{Shtessel2009}, as well as feedback linearization-based strategies \cite{Yang2013,Santos2025}. The reader is refered to \cite{Aguiar2007}, \cite{Casau20215} and the references therein for a deeper understanding on trajectory tracking control techniques for underactuated mechanical systems.

Our strategy leverages the tools of adaptive backstepping to derive a control law that defines the desired thrust-vector and ensures almost-global trajectory tracking in the presence of a constant exogenous disturbance. Moreover, borrowing from the literature on attitude control of underactuated rigid-bodies and geometric control on the sphere \cite{Gamagedara2019,Bullo1995}, we rely on a reduced attitude representation on the 2-sphere to fully exploit the axial symmetry of the vehicle and obtain almost-global attitude tracking with decoupled yaw motion. Additionally, the control point is moved from the center of mass to the center of oscillation in order to cancel small torque-to-force couplings and prevent non-minimum phase behavior, as originally detailed in \cite{Murray1995} and \cite{Nieuwstadt1995}. The backstepping technique and the reduced attitude representation are also used for quadrotor control to ensure that it tracks the desired thrust-vector.

\subsection{Related Work}

Other works have focused on developing low-cost and low-risk thrust vectoring testbeds to validate novel Guidance, Navigation \& Control (GNC) algorithms. Two independent projects from Switzerland, \cite{Spannagl2021}, \cite{Linsen2022}, have built and successfully tested thrust vectoring flying platforms relying on electric brushless motor-driven contra-rotating propellers attached to a gimbal mechanism. Both use Model Predictive Control (MPC) to solve the trajectory planning and tracking problem while optimizing actuation effort and satisfying constraints. Contrary to these works, in our platform thrust vectoring is achieved through multirotor attitude and force generation, avoiding gimbal-specific constraints such as servo angle/rate limits and backlash, and enabling rapid, repeatable experiments through readily available multirotor hardware. Moreover, while \cite{Spannagl2021,Linsen2022} primarily validate optimization-based GNC (optimal guidance and MPC/NMPC-type tracking), our paper instead focuses on a closed-form nonlinear adaptive controller targeting stability/robustness guarantees without requiring online optimization.   

As for quadrotor-based solutions, previous work from some of the authors has solved the trajectory tracking problem for the flying inverted pendulum via adaptive backstepping \cite{Yang2024}. In this case, the pendulum is taken as a point mass, with the quadrotor being the main contributor to the overall mass of the system. A different quadrotor-based scheme has been proposed in \cite{Elke2024}, where the primary goal is to accurately represent the dynamics of a launch vehicle. Both a flexible inverted pendulum and hanging mass are attached to the quadrotor to simulate aerodynamic instability and fuel sloshing. As opposed to our platform, the quadrotor is the main body and the inverted pendulum is used to replicate launch vehicle modes. Control design also differs, as a standard Linear Quadratic Regulator (LQR) is derived on a linear representation of the system to track attitude commands from an independent guidance loop.   

Focusing on the control strategy, nonlinear/geometric adaptive control techniques are standard when trying to enforce strict stability guarantees in the presence of constant disturbances \cite{Astolfi2008}. Relevant works which share similar design methodology can be found in \cite{Hua2009,Yang2024,Mofid2022,Reis2023b,Sun2015}. In our paper, we apply these techniques to an original model of a thrust-vector-controlled rocket-like vehicle, relying on a reduced-attitude representation on the 2-sphere. In doing so, we build on a considerable gap in the rocket trajectory control literature, where tracking solutions with (near) global performance are lacking \cite{Chai2021}.
	
Reduced-attitude representations have long been used in rigid-body attitude control, motivated by the need to regulate a body-fixed axis and by the presence of torque underactuation.     Seminal work can be found in \cite{Bullo1995}, \cite{Tsiotras1994}. More recently, these techniques have been exploited in the context of trajectory tracking with aerial vehicles for which the thrust direction governs position tracking independently of yaw motion \cite{Brescianini2020}, most predominantly quadrotor UAVs \cite{Gamagedara2019}, \cite{Kooijman2019,Coates2020,Madeiras2024,Montanez2024}. In our paper, we extend this methodology to a class of axissymetric aerial vehicles controlled by means of thrust vectoring. We rely on a reduced-attitude error that respects the geometrical properties of the 2-sphere, similar to the one employed in \cite{Gamagedara2019}, \cite{Montanez2024}, and propose a geometric tracking law derived through backstepping that ensures almost global asymptotic convergence. As opposed to previous works, to ensure stability of the closed-loop system we do not rely on time-scale separation \cite{Kooijman2019}, and do not restrict the admissible gains \cite{Gamagedara2019}, \cite{Montanez2024} or the reduced-attitude configuration space \cite{Madeiras2024}. This is achieved through a sequential backstepping design that leads to a single Lyapunov function encapsulating all tracking errors (position and attitude).

\subsection{Contributions}

Contributions of this work can be segmented into (i) platform design and experimental realization, and (ii) modeling and nonlinear control development:
\begin{enumerate}
	\item To the best of our knowledge, the \textit{QuadRocket} is the first quadrotor-based rocket prototype where the largest mass contributor is not the quadrotor but a body of non-negligible inertia attached to it through an universal joint;
	\item We demonstrate that a quadrotor can act as a thrust-vector actuator, enabling control of a flying inverted pendulum when modeled as a thrust-vector-controlled rocket;
	
\item An original model for the dynamics and kinematics of a generic axisymmetric rigid body controlled via thrust vectoring is derived by exploiting a reduced attitude representation on the 2-sphere and considering a control point transformation that cancels out non-minimum phase behavior.

 \item A novel almost-global trajectory tracking solution in the presence of a constant exogenous disturbance is obtained. For this class of vehicles, our approach is the first to merge position and attitude control in a single loop while using a reduced attitude error that respects the geometrical properties of the 2-sphere.

\end{enumerate}
\subsection{Notation}

Throughout this paper, the following notation applies. Bold lowercase and uppercase symbols stand for column vectors and matrices, respectively. The $n$-dimensional Euclidean space is represented by ${\mathbb{R}}^n$, with $\mathbb{S}^n$ denoting the set of unit vectors in $\mathbb{R}^{n+1}$. The space tangent to the $n$-sphere $\mathbb{S}^n$ at a point $\b{p} \in \mathbb{S}^n$ is denoted by $T_\b{p}\mathbb{S}^n$. The set of positive real numbers is denoted by $\mathbb{R}^+$, while $\mathbb{R}^+_0$ represents the set of non-negative real numbers. The symbol $\mathbf{I}$ denotes the identity matrix, and $\b{0}$ denotes a vector/matrix of zeros, both of appropriate dimensions. The transpose operator is denoted by $(\cdot)^\mathsf{T}$. The symbol $\bs{\otimes}$ stands for the Kronecker product while the direct sum is represented by $\bs{\oplus}$. Given a vector $\b{x} = \left[\:x_1\:\:x_2\:\dots\:x_n\:\right]^\mathsf{T}\in \mathbb{R}^n$, its Euclidean norm is defined as $||\b{x}|| \coloneq \sqrt{\b{x}^\mathsf{T}\bf{x}}$. A function $f$ is of class $\mathcal{C}^n$ if its derivatives $f'$, $f''$, ...,$f^{(n)}$ exist and are continuous.  The special orthogonal group of order three is defined as SO(3)$\coloneq\left\{\b{X}\in \mathbb{R}^{3\times3}: \b{X}\b{X}^\mathsf{T}=\b{X}^\mathsf{T} \b{X} = \b{I},\,\text{det}(\b{X})=1\right\}$. The operator $\b{S}(\cdot):\mathbb{R}^3\mapsto \mathbb{R}^{3\times3}$ yields a skew-symmetric matrix satisfying (the cross product) $\b{S}(\b{x})\b{y} =\b{x} \times \b{y} $, for any $\b{x},\,\b{y}\in \mathbb{R}^3$. The operator $\text{diag}(\b{x}):\mathbb{R}^n\mapsto \mathbb{R}^{n\times n}$ returns a diagonal matrix with the elements of $\b{x}$ along the main diagonal. The vectors $\b{e}_1,\,\b{e}_2,\,\b{e}_3 \in \mathbb{S}^2$ are orthonormal and form a basis of $\mathbb{R}^3$, such that the identity matrix can be expressed as ${\b{I}}=\left[{\b{e}}_1\;\b{e}_2\;\b{e}_3\right]$. 

\begin{figure}[t]
	\centering
	\includegraphics{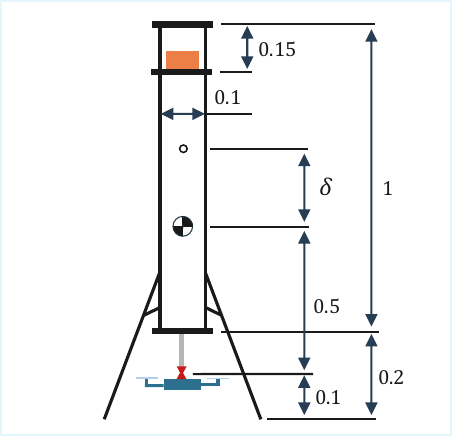} 
	\caption{\textit{QuadRocket} scheme. All dimensions in meters. The orange rectangle represents the battery, and the universal joint is shown in red.}
	\label{fig:quadrocket_scheme}
\end{figure}

\section{SYSTEM OVERVIEW}

The \textit{QuadRocket} platform was conceived as an aerial robotic testbed that is inexpensive, modular, and easy to reproduce, while still capturing the key dynamics of thrust-vector-controlled rocket-like vehicles. The design relies almost exclusively on off-the-shelf components and simple custom parts so that the platform can be replicated and modified in typical robotics laboratories. The system consists of a quadrotor UAV mechanically coupled to a cylindrical body through a universal joint, forming a flying inverted pendulum with non-negligible inertia, as illustrated in Fig.~\ref{fig:quadrocket_scheme}. This configuration preserves the essential coupling between thrust-vector direction and rigid-body attitude, enabling realistic evaluation of guidance and control strategies for thrust-vector-controlled vehicles. The cylindrical body is assembled from four independent carbon-fiber tubes with carbon-fiber disks attached to both ends, as shown in Fig.~\ref{fig:vehicle_parts}, each tube carrying a carbon-fiber leg so that the vehicle can stand upright on the ground. An additional carbon disk is positioned near the top of the structure to support the battery that powers the quadrotor. Positioning the battery at this location shifts the overall center of mass closer to the top, which increases the torque lever arm and, consequently, control authority.

An off-the-shelf aluminum universal joint is attached to the top of the quadrotor using a 3D printed PLA (Polylactic Acid) part (see Fig. \ref{fig:vehicle_parts}). To connect the carbon fiber cylindrical structure to the top part of the universal joint, an aluminum shaft was first designed and then manufactured by an external partner. The shaft has a threaded end so that it can be secured to the bottom carbon plate using a nut, while connection to the universal joint is guaranteed by friction, aided by a tightening shaft built into the joint. The universal joint allows for unconstrained rotational motion between both bodies in two degrees of freedom, corresponding to the gimbal angles, up to 40$\,^\circ$ of maximum tilt. A rotational constraint is imposed on the remaining degree of freedom by the universal joint, meaning that the quadrotor and the main body will have the same rotation on their individual axis of connection to the joint. The mass properties of the system are collected on Table \ref{tab:parameters}.
\begin{table}[h]
	\caption{System mass properties.}
	\label{tab:parameters}
	\centering
	\renewcommand{\arraystretch}{1.2} 
	\setlength{\tabcolsep}{8pt} 
	\begin{tabular}{>{\columncolor{gray!8}}c c c}
		\rowcolor{gray!20}
		\textbf{Component} & \textbf{Mass (kg)} & \textbf{\%Total}\\
		\hline
		Quadrotor&  0.45 & 29.0\\
		Carbon fiber structure & 0.64 & 41.3\\
		Universal joint and shaft & 0.11 & 7.1\\
		Battery and power cables & 0.35 & 22.6\\
		\hline
		Fully assembled vehicle & 1.55 & 100\\
		\hline
	\end{tabular}
\end{table}

\subsection{Quadrotor}

\begin{figure}[t]
	\centering
	\includegraphics[width=\columnwidth]{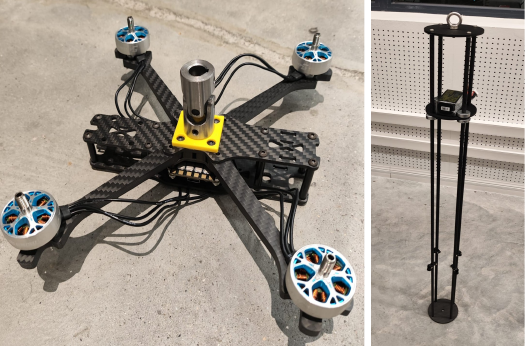} 
	\caption{Quadrotor with mounted universal joint and main body.}
	\label{fig:vehicle_parts}
\end{figure}

The quadrotor had been previously designed and assembled at the SCORE lab of the University of Macau and is of the FPV racing type (see Fig. \ref{fig:vehicle_parts}). Its lightweight structure is made out of carbon fiber plates connected through aluminum spacers and carbon fiber arms to support each of its four brushless motors. The BetaFPV Snow 2306 2500KV motors were used with 5 inch propellers, producing a maximum combined thrust of 55$\,$N at a current of 160$\,$A. To power the drone, a 4-cell LiPo battery is used with a discharge rate of 120$\,$C and a nominal capacity of 2000$\,$mA.h. As for the flight computer, a BrainFPV RADIX 2 HD Flight Controller is used with BetaFlight 4.5.1 flight software running on it. For our implementation, the flight software is configured to manual mode such that the quadrotor tracks thrust and angular velocity commands coming from an RC communication link. This RC communication link will be directly commanded by the control algorithm to be derived.

\section{PROBLEM FORMULATION}

The \textit{QuadRocket} consists of two interconnected rigid bodies linked by a universal joint. Directly modeling and controlling this coupled multibody system with joint constraints leads to a high-dimensional and cumbersome representation that obscures the main control challenges. To simplify the flight-stabilization and control problem, we adopt an abstraction in which the vehicle is modeled as a single axisymmetric rigid body actuated by a vectored force applied along its longitudinal axis, i.e., the quadrotor is treated as a thrust-vector actuator:
\begin{assumption}\label{ass:single_body}
	The \textit{QuadRocket} is an axisymmetric rigid body actuated by a vectored force applied along its longitudinal axis. The magnitude of the control force is given by the thrust produced by the quadrotor, and its direction is determined by the relative orientation between the quadrotor and the cylindrical structure. 
\end{assumption}
Under Assumption \ref{ass:single_body}, the cylindrical body defines the inertial position, velocity, and orientation of the vehicle, while the abstracted rigid body has the combined mass and approximate inertia properties of both physical bodies. 
\begin{assumption}\label{ass:quad_omega}
The angular velocity of the quadrotor is regulated by a sufficiently fast inner-loop controller. Consequently, it can be treated as a control input.
\end{assumption}

Assumptions \ref{ass:single_body} and \ref{ass:quad_omega} allow us to decompose the control problem into two complementary tasks:
\begin{enumerate}
	\item Design a flight controller for a generic axisymmetric rigid body actuated by a vectored force, applied on its longitudinal axis below the center of mass;
	\item Design an attitude controller to command the angular velocity of the quadrotor such that it tracks the desired thrust-vector input computed by the flight controller.
\end{enumerate}

\subsection{\textit{QuadRocket} Dynamic Modelling} \label{sec:dyn_model}

Following Assumption \ref{ass:single_body}, the \textit{QuadRocket} can be seen as an underactuated vehicle with constant mass, $m \in \mathbb R^{+}$, while its inertia tensor is the positive definite matrix given by $\b{J}\in \mathbb{R}^{3\times3}$. Let us introduce three reference frames: an inertial one denoted by \{$I$\}, a body-fixed one denoted by \{$B$\}, whose origin coincides with the center of mass of the vehicle, and a third one fixed to the quadrotor denoted by \{$Q$\}. Figure \ref{fig:quadrocket} depicts the reference frames.

Let $\b{p}, \b{v} \in \mathbb{R}^3$ respectively stand for the inertial position and velocity of the center of mass, $\b{R}\in$ SO(3) be the rotation matrix from \{$B$\} to \{$I$\}, and $\bs{\omega}\in\mathbb{R}^3$ be the angular velocity of $\{B\}$ with respect to $\{I\}$, expressed in $\{B\}$. Then, the kinematics of the \textit{QuadRocket} are given by
\begin{subequations}\label{eq:kinematics}
	\begin{empheq}[left=\empheqlbrace]{align}
		\dot{\b{p}} &= \b{v}\,,\\[4pt]
		\dot{\b{R}} &= \mathbf{R}\b{S}(\bs\omega)\,.
	\end{empheq}
\end{subequations}
Moreover, let $\b{R}_q\in$ SO(3) be the rotation matrix from \{$Q$\} to \{$I$\}, and $\bs{\omega}_q\in\mathbb{R}^3$ be the angular velocity of $\{Q\}$ with respect to $\{I\}$, expressed in $\{Q\}$. Then, the attitude kinematics of the quadrotor are given by
\begin{equation}\label{eq:quad_kinematics}
	\dot{\b{R}}_q = \b{R}_q\b{S}(\bs{\omega}_q)\,.
\end{equation}

By defining a virtual thrust-vector control input, denoted by $\b{u} \in \mathbb{R}^3$, as the instantaneous force vector produced by the quadrotor expressed in \{$B$\}, i,e,
\begin{equation}\label{eq:tvcinput}
	\b{u} \coloneq -T\b{R}^\mathsf{T}\b{R}_q\b{e}_3\,,
\end{equation}
where $T \in \mathbb{R}^+_0$ is the thrust magnitude, the dynamics of the vehicle are obtained:
\begin{subequations}\label{eq:dynamics}
	\begin{empheq}[left=\empheqlbrace]{align}
		\dot{\b{v}} &= g\b{e}_3 + \frac{1}{m}\b{R}\mathbf{u} +\b{b}\,,\\[2pt]
		\dot{\bs{\omega}} &= -\b{J}^{-1}\b{S}(\bs{\omega})\b{J}\bs{\omega} + L\b{J}^{-1}\b{S}(\b{e}_3)\b{u}\,,
	\end{empheq}
\end{subequations}
where $g\in\mathbb{R}^+$ is the gravitational acceleration, $L \in \mathbb R^+$ is the distance between the center of mass and the universal joint, both assumed to be located on the longitudinal axis, and $\b{b} \in \mathbb{R}^3$ is an unknown disturbance.
\begin{assumption}\label{ass:disturb}
	The unknown exogenous disturbance, $\b{b}$, is constant or slowly time-varying, i.e, $\dot{\b{b}}\approx\b{0}$.
\end{assumption}

\subsection{Reduced attitude Representation in $\mathbb{S}^2$}
Considering the nature of thrust-vector actuation, the control task, and the axial symmetry of the vehicle, it is advantageous to decouple the rotational motion of the longitudinal axis (pitch and roll) from rotational motion around this axis (yaw). To do so, we first rewrite the attitude kinematics of the vehicle by applying the projection map \cite{Bullo1995}
\begin{equation}\label{eq:projector}
	\bs{\pi}_i(\b{R}) \coloneq \b{R}\b{e}_i\,,
\end{equation}
which takes the full attitude of the vehicle given by its rotation matrix, $\b{R} \in$ SO(3), and projects it into a unit vector in $\mathbb{S}^2$.
Using this map, the rotation matrix can be expressed as	$\b{R} = \left[\,\b{r}_1\;\b{r}_2\;\b{r}_3\,\right]$, where
\begin{equation}
	\b{r}_i \coloneq  \bs{\pi}_i(\b{R})\in \mathbb{S}^2\,,\hspace{10pt}i = \{1,\,2,\,3\}\,.
\end{equation}
The unit vector $\b{r}_3$ defines the direction of the longitudinal axis, i.e., the $z$-axis of $\{B\}$, expressed in \{$I$\} and its kinematics are given by
\begin{equation}
	\dot{\b{r}}_3 = -\b{S}(\b{r}_3 )\,\bs{\Omega}_{12}\,,\hspace{10pt}	\dot{\b{r}}_3  \in T_{\b{r}_3}\mathbb{S}^2\subset \mathbb{R}^3\,,
\end{equation} 
where 
\begin{equation}\label{eq:om12}
	\bs{\Omega}_{12} \coloneq -\b{S}^2(\b{r}_3)\,\bs{\Omega} = \omega_1\b{r}_1 + \omega_2 \b{r}_2\,,\hspace{10pt}		\bs{\Omega}_{12}  \in T_{\b{r}_3}\mathbb{S}^2\,,
\end{equation}
with
\begin{equation}\label{eq:om}
	\bs{\Omega} \coloneq \b{R}\,\bs{\omega}\,.
\end{equation}

\begin{assumption}\label{ass:diag}
	  Due to the axial symmetry of the vehicle, its inertia tensor is taken as a diagonal matrix of the form $\b{J} = \text{diag}(j_\perp,j_\perp,j_\parallel)$, where $j_\perp,\,j_\parallel \in\mathbb{R}^+$ are the transverse and longitudinal inertia values, respectively.
\end{assumption}
Using the reduced attitude representation, together with Assumption \ref{ass:diag}, the dynamics and kinematics take the new form
\begin{subequations}\label{eq:dynamics_r3}
	\begin{empheq}[left=\empheqlbrace]{align}
		&\dot{\b{p}} = \b{v}\,,\\[2pt]
		&\dot{\b{v}} = g\b{e}_3 + \frac{1}{m}u_3\b{r}_3  -\frac{1}{mL}\b{S}(\b{r}_3)\bs{\Gamma} + \b{b}\,,\\[2pt]
		&\dot{\b{r}}_3 = -\b{S}(\b{r}_3 )\,\bs{\Omega}_{12}\,,\\[2pt]
		&\dot{\bs{\Omega}}_{12} = \frac{1}{j_\perp}\bs{\Gamma}_c\,,
	\end{empheq}
\end{subequations}
where $\bs{\Gamma}\in T_{\b{r}_3}\mathbb{S}^2$ is the torque input caused by quadrotor thrust vectoring, defined as
\begin{equation}\label{eq:Gamma}
	\bs\Gamma \coloneq  L\b{S}(\b{r}_3)\b{R}\b{u}\,,
\end{equation}
and $\bs{\Gamma}_c\in T_{\b{r}_3}\mathbb{S}^2$ is an auxiliary torque input that gathers the thrust vectoring, $\bs{\Gamma}$, and gyroscopic terms, $\bs{\Gamma}_0\in T_{\b{r}_3}\mathbb{S}^2$, defined as
\begin{equation}\label{eq:Gamma_c}
	\bs\Gamma_c \coloneq \bs\Gamma + \bs\Gamma_0\,,
\end{equation}
with
\begin{equation}\label{eq:gamma_0}
	\bs{\Gamma}_0 \coloneq j_\perp\,\b{S}^2(\b{r}_3)\left(\b{R}\b{J^{-1}}\b{S}(\bs\omega)\b{J}\bs\omega - \b{S}(\bs{\Omega})\bs{\Omega}_{12}\right)\,.
\end{equation}

\subsection{Non-minimum Phase Behavior and Coordinate Change}
Looking at the vehicle dynamics and kinematics in (\ref{eq:dynamics_r3}), besides the problem of underactuation, a coupling between force and torque actuation is noted through the dependency of the position dynamics on the torque input $\bs\Gamma$. This is characteristic to underactuated vehicles controlled by means of a vectored force and is known to cause undesired non-minimum phase behavior \cite{Pflimlin2004}. To tackle this problem, and similarly to what was done in \cite{Reis2023,Cabecinhas2018,Pflimlin2004} for other autonomous platforms controlled by a vectored force, let us write the dynamics and kinematics of the system with respect to an arbitrary point located on the longitudinal axis of the vehicle at a given distance, denoted by $\delta \in \mathbb{R}^{+}$, above its center of mass, as illustrated in Fig. \ref{fig:quadrocket_scheme}. The inertial position and velocity of this point, respectively $\b{p}_\delta,\,\b{v}_\delta \in \mathbb{R}^3$, are given by
\begin{align}\label{eq:pv_delta}
	&\b{p}_\delta =\b{p} - \delta\b{r}_3 \,,\\[5pt]
	&\b{v}_\delta = \b{v} +\delta\,\b{S}(\b{r}_3 )\,\bs{\Omega}_{12}\,,
\end{align}
leading to the following position dynamics and kinematics:
\begin{subequations}
	\begin{align}
			&\dot{\b{p}}_\delta = \b{v}_\delta = \b{v} - \delta\,\dot{\b{r}}_3\,,\label{eq:kin_delta}\\
			&	\dot{\b{v}}_\delta = g\b{e}_3 + \overline{u_3}\,\b{r}_3 + \b{S}(\b{r}_3 )\left(\frac{\delta}{j_\perp}\,\bs{\Gamma}_c - \frac{1}{mL}\bs{\Gamma}\right) + \b{b}\,,\label{eq:dyn_delta}
	\end{align}
\end{subequations}
where
\begin{equation}\label{eq:u3_bar}
	\overline{u_3}\coloneq \frac{u_3}{m} + \delta\,\bs{\Omega}^\mathsf{T}_{12}\bs{\Omega}_{12}\,.
\end{equation} 
Based on (\ref{eq:dyn_delta}), it is straightforward to infer that the coupling term coming from the torque input can be canceled with the additional cross-product term arising from the coordinate change if
\begin{equation}\label{eq:delta_condition}
	\delta = \displaystyle\frac{j_\perp}{mL}\:\:\wedge\:\: \bs{\Gamma} = \bs{\Gamma}_c\,.
\end{equation}
The value of $\delta$ is user-defined, corresponding to a control point shift to the center of oscillation \cite{Murray1995}. On the other hand, the condition on the torque input is equivalent to $\bs{\Gamma}_0=\b{0}$, which, from (\ref{eq:gamma_0}), is verified for $\omega_3 = 0$. Considering that rotational motion around $\b{r}_3$ is decoupled from the remaining rotational degrees of freedom, and that $\omega_3$ can be set by commanding the angular velocity of the quadrotor accordingly, we assume that this condition can also be satisfied. Then, the model used for control design takes the final form
\begin{subequations}\label{eq:dynamics_delta_r3}
	\begin{empheq}[left=\empheqlbrace]{align}
		&\dot{\b{p}}_\delta = \b{v}_\delta\,,\\[2pt]
		&\dot{\b{v}}_\delta = g\b{e}_3 + \overline{u_3}\,\b{r}_3 + \b{b}\,,\\[2pt]
		&\dot{\b{r}}_3 = -\b{S}(\b{r}_3 )\,\bs{\Omega}_{12}\,,\label{eq:r3_dot_eq}\\[2pt]
		&\dot{\bs{\Omega}}_{12} = \frac{1}{j_\perp}\bs{\Gamma}_c\,.
	\end{empheq}
\end{subequations}

\subsection{Problem Statement}
We are now in the position to formally state the two problems that naturally arise from the complementary control tasks:
\begin{problem}[Virtual Thrust-Vector Control]\label{prob:virtual_tvc}
	Given a time-parameterized reference trajectory denoted by $\b{p}_d(t) \in \mathbb{R}^3$, of class $\mathcal{C}^4$, and whose time derivatives are bounded, design a control law for the virtual thrust-vector input $\b{u}$ in \eqref{eq:tvcinput} that, under the dynamic model in \eqref{eq:dynamics_delta_r3}, ensures convergence of the tracking error $\b{p}_\delta - \b{p}_d$ to zero in the presence of unknown constant disturbances, $\b{b}$.
\end{problem}

\begin{problem}[Quadrotor Attitude Control]\label{prob:quadrotor_control}
	 Derive a control law for the angular velocity $\bs{\omega}_q$ in (2) such that the quadrotor is able to track the commanded virtual thrust-vector input $\b{u}$.  
\end{problem}

\section{NONLINEAR CONTROL DESIGN}

In this section, strategies to independently solve Problems \ref{prob:virtual_tvc} and \ref{prob:quadrotor_control} will be derived by leveraging the tools of nonlinear adaptive backstepping and dynamic surface control.

\subsection{Virtual Thrust-Vector Control}

Given the model in (\ref{eq:dynamics_delta_r3}), we shall design a control law for the virtual thrust-vector input, $\b{u}(t)$, that solves Problem \ref{prob:virtual_tvc}. For that purpose, let us start by defining the position and velocity tacking errors as the first states of our error system:
\begin{equation}
	\b{z}_1 \coloneq \b{p}_\delta - \b{p}_d \in \mathbb{R}^3\,,
\end{equation}
\begin{equation}\label{eq:z2}
	\b{z}_2 \coloneq \b{v}_\delta - \dot{\b{p}}_d \in \mathbb{R}^3\,.
\end{equation}

Consider now the first candidate Lyapunov function given by 
\begin{equation}
	V_1 \coloneq \frac{1}{2}\b{z}^\mathsf{T}_1\b{z}_1\,.
\end{equation}
Its time derivative can be written as
\begin{equation}\label{eq:v1_derivative}
	\dot{V}_1 = -W_1 + \b{z}^\mathsf{T}_1\left(\b{z}_2 + \b{K_p}\b{z}_1\right)\,,
\end{equation}
where $ \b{K_p}\in \mathbb{R}^{3\times3}$ is a positive definite gain matrix, and $W_1 \coloneq \b{z}^\mathsf{T}_1\b{K_p}\b{z}_1 \in \mathbb{R}^+_0$ is a non-negative auxiliary term. Looking at (\ref{eq:v1_derivative}), a new error $\b{e} \coloneq \b{z}_2 + \b{K_p}\b{z}_1 \in \mathbb{R}^3$ is defined, leading to our second candidate Lyapunov function:
\begin{equation}\label{eq:V2}
	V_2 \coloneq V_1 + \frac{1}{2}\b{e}^\mathsf{T}\b{e}\,.
\end{equation}
Using the dynamics given in (\ref{eq:dynamics_delta_r3}), the time derivative of (\ref{eq:V2}) can be expressed as
\begin{equation} \label{eq:V2_derivative}
	\dot{V}_2 = - W_2 + \b{e}^\mathsf{T}\left( \bs{\zeta} + g\b{e}_3 + \b{b} - \ddot{\b{p}}_d + \overline{u_3}\,\b{r}_3\right),
\end{equation}
where $W_2\coloneq W_1 + \b{e}^\mathsf{T}\b{K_v}\b{e} \in \mathbb{R}^{+}_{0}$ is another non-negative auxiliary term, with $\b{K_v} \in \mathbb{R}^{3\times3}$ being a positive definite symmetric gain matrix, and where
\begin{equation}\label{eq:zeta}
	\bs{\zeta} \coloneq \left(\b{I} + \b{K_v}\b{K_p} \right)\b{z}_1 + \left(\b{K_p} + \b{K_v} \right)\b{z}_2 \in \mathbb{R}^3
\end{equation} 
is an auxiliary vector, defined to simplify notation. To cancel out the last term of (\ref{eq:V2_derivative}), let us define a vector which gathers the known terms as
\begin{equation}\label{eq:xi}
	\bs{\xi} \coloneq \bs{\zeta} +
	g\b{e}_3 + \widehat{\b{b}}_1 - \ddot{\b{p}}_d \in \mathbb{R}^3\,,
\end{equation}
where $\b{\widehat{b}}_1 \in \mathbb{R}^3$ is a first estimate on the constant unknown disturbance. Accordingly, the estimation error is defined as 
\begin{equation}\label{eq:b1_tilde}
	\widetilde{\b{b}}_1 \coloneq \b{\widehat{b}}_1 - \b{b} \in \mathbb{R}^3\,.
\end{equation}
By replacing (\ref{eq:xi}) and (\ref{eq:b1_tilde}) in (\ref{eq:V2_derivative}), we get
\begin{equation} \label{eq:V2_derivative2}
	\dot{V}_2 = - W_2 + \b{e}^\mathsf{T}\left( \bs{\xi} + \overline{u_3}\,\b{r}_3 \right)  - \b{e}^\mathsf{T}\widetilde{\b{b}}_1  \,.
\end{equation}

Note that the second term in (\ref{eq:V2_derivative2}) cannot be directly canceled out due to the underactuated nature of the system. Indeed, nothing imposes alignment between $\bs\xi$ and $\b{r}_3$. Thus, to tackle this problem, a desired direction for the longitudinal axis is defined as
\begin{equation}\label{eq:r3d}
	\b{r}_{3d} \coloneq \frac{\bs{\xi}}{||\bs{\xi}||}\,,\hspace{10pt}\bs\xi\neq \b0\,,
\end{equation}
and the following law for the longitudinal component of the virtual thrust-vector is imposed:
\begin{equation}\label{eq:u3_bar_law}
	\overline{u_3} \coloneq -\bs{\xi}^\mathsf{T}\b{r}_3\,.
\end{equation}
Replacing both definitions, (\ref{eq:r3d}) and (\ref{eq:u3_bar_law}), in (\ref{eq:V2_derivative2}), we obtain
\begin{equation} \label{eq:V2_derivative3}
	\dot{V}_2 = - W_2 + \Delta   - \b{e}^\mathsf{T}\widetilde{\b{b}}_1  \,,
\end{equation}
with
\begin{equation}
	\Delta \coloneq ||\bs{\xi}||\b{e}^\mathsf{T}\b{S}(\b{r}_3)\b{z_r} \in \mathbb{R}\,,
\end{equation}
where a reduced attitude error
\begin{equation}\label{eq:attitude_err}
	\b{z_r} \coloneq \b{S}(\b{r}_{3d})\b{r}_3 \in T_{\b{r}_{3}}\mathbb{S}^2 \cap  T_{\b{r}_{3d}}\mathbb{S}^2
\end{equation}
is defined. This error definition respects the geometrical properties of $\mathbb{S}^2$ -- the space where the reduced attitude representation lives -- allowing for a more intuitive control design. 

To account for the last term in (\ref{eq:V2_derivative3}), a third candidate Lyapunov function is defined as
\begin{equation}\label{eq:V3}
	V_3 \coloneq V_2 + \frac{1}{2}\widetilde{\b{b}}_1^\mathsf{T}\bs{\Lambda}^{-1}_1\widetilde{\b{b}}_1\,,
\end{equation}
where $\bs{\Lambda}_1 \in \mathbb{R}^{3\times3}$ is a positive definite auxiliary gain matrix. The time derivative of (\ref{eq:V3}) can be written as
\begin{equation}\label{eq:V3_derivative}
	\dot{V}_3 = -W_2 + \Delta + \widetilde{\b{b}}_1^\mathsf{T}\left(\bs{\Lambda}_1^{-1}\dot{\widetilde{\b{b}}}_1 - \b{e} \right)\,,
\end{equation}
thus, if the dynamics of the estimation error are given by
\begin{equation}\label{eq:b1dot_tilde}
	\dot{\widetilde{\b{b}}}_1 = \bs{\Lambda}_1\b{e}\,,
\end{equation}
the time derivative in (\ref{eq:V3_derivative}) simplifies to
\begin{equation}\label{eq:V3_derivative2}
	\dot{V}_3 = -W_2 + \Delta\,.
\end{equation}

Under Assumption \ref{ass:disturb}, the estimation error dynamics in \eqref{eq:b1dot_tilde} can be enforced through the adaptation law
\begin{equation}\label{eq:adapt_1}
	\dot{\widehat{\b{b}}}_1 = \bs{\Lambda}_1\b{e}\,,
\end{equation} 
which will correspond to an integral action that grants the controller robustness to constant or slowly varying disturbances.

At this point, we must ensure that $\b{r}_3$ aligns with $\b{r}_{3d}$, thereby canceling out the term $\Delta$ in (\ref{eq:V3_derivative2}). In other words, a control law for the torques produced by the virtual thrust-vector must be devised such that reduced attitude control is achieved. With that goal in mind, we introduce yet another candidate Lyapunov function
\begin{equation}\label{eq:V4}
	V_4 \coloneq V_3 + h_r\,\left(1-\b{r}_{3d}^\mathsf{T}\b{r}_3\right)\,,
\end{equation}  
where $h_r \in \mathbb{R}^{+}$ is a positive auxiliary tuning parameter. The time derivative of (\ref{eq:V4}) can be written as
\begin{equation}\label{eq:V4_derivative}
	\dot{V}_4 = -W_3 + \Delta + \b{z_r}^\mathsf{T}\left[k_r\b{z_r} + h_r\left(\bs{\Omega}_{12} - \bs{\Omega}_{12d} \right)\right]\,,
\end{equation}
where $W_3 \coloneq W_1 + W_2 + k_r \b{z_r}^\mathsf{T}\b{z_r} \in \mathbb{R}^{+}_{0}$ is a third non-negative auxiliary term, $k_r \in \mathbb{R}^{+}$ is a positive scalar gain, and $\bs{\Omega}_{12d} \coloneq \b{S}(\b{r}_{3d})\dot{\b{r}}_{3d} \in T_{\b{r}_{3d}}\mathbb{S}^2$ is the desired inertial angular velocity projected onto the plane normal to $\b{r}_{3d}$. Let us now define a commanded angular velocity $\bs{\Omega}_{12c} \in T_{\b{r}_{3}}\mathbb{S}^2$ as
\begin{equation}\label{eq:om12c}
	\bs{\Omega}_{12c} \coloneq -\frac{k_r}{h_r}\b{z_r} + \frac{1}{h_r}||\bs{\xi}||\b{S}(\b{r}_3)\b{e} - \b{S}^2(\b{r}_3)\widehat{\bs{\Omega}}_{12d}\,,
\end{equation}
where $\widehat{\bs{\Omega}}_{12d} \in \mathbb{R}^3$ (shown in Appendix \ref{app:aux_var}) is the known part of $\b{\bs{\Omega}}_{12d}$, computed using a second estimate on $\b{b}$ defined as $\widehat{\b{b}}_2 \in \mathbb{R}^3$, such that
\begin{equation}\label{eq:om12dhat}
	\widehat{\bs{\Omega}}_{12d} =	\bs{\Omega}_{12d}   + ||\bs{\xi}||^{-1}\b{S}(\b{r}_{3d})\left(\b{K_p} + \b{K_v}\right)\,\widetilde{\b{b}}_2\,.
\end{equation}
The associated angular velocity error is defined as
\begin{equation}\label{eq:ang_vel_err}
	\b{z}_{\bs{\Omega}} \coloneq \bs{\Omega}_{12} - \bs{\Omega}_{12c} \in T_{\b{r}_{3}}\mathbb{S}^2\,.
\end{equation} 
Then, through substitution and simple algebraic manipulation, we can write (\ref{eq:V4_derivative}) as
\begin{equation}\label{eq:V4_derivative2}
	\dot{V}_4 = -W_3 + h_r\b{z}^\mathsf{T}_{\b{r}} \left(\b{z}_{\bs{\Omega}} + \bs{\Xi}\,\widetilde{\b{b}}_2 \right)\,,
\end{equation}
with 
\begin{equation}\label{eq:Xi}
	\bs{\Xi} \coloneq -||\bs{\xi}||^{-1}\b{S}^2(\b{r}_{3})\b{S}(\b{r}_{3d})\left(\b{K_p} + \b{K_v}\right) \in \mathbb{R}^{3\times3}\,.
\end{equation}
Our goal now is to render $\dot{V}_4$ negative by canceling out the term projected onto ${\bf z_r}$ in (\ref{eq:V4_derivative2}). To that end, we start by deriving an appropriate adaptation law for the second estimate on the unknown disturbance, $\widehat{\b{b}}_2$. Similarly to what was done in (\ref{eq:V3}) for the first estimate, we define a fifth candidate Lyapunov function as
\begin{equation}\label{eq:V5}
	V_5 \coloneq V_4 + \frac{1}{2}\widetilde{\b{b}}_2^\mathsf{T}\bs{\Lambda}^{-1}_2\widetilde{\b{b}}_2\,,
\end{equation}
where $\bs{\Lambda}_2 \in \mathbb{R}^{3\times3}$ is a positive definite auxiliary gain matrix. The time derivative of (\ref{eq:V5}) can be written as 
\begin{equation}
	\dot{V}_5 = -W_3 +  h_r\b{z}^\mathsf{T}_{\b{r}}\b{z}_{\bs{\Omega}} + \widetilde{\b{b}}_2^\mathsf{T}\left(\bs{\Lambda}_2^{-1}\dot{\widetilde{\b{b}}}_2 +h_r\, \bs{\Xi}^\mathsf{T} \b{z_r} \right)\,.
\end{equation}
Therefore, the adaptation law
\begin{equation}\label{eq:adapt2}
	\dot{\widehat{\b{b}}}_2 = -h_r\bs{\Lambda}_2\,\bs{\Xi}^\mathsf{T} \b{z_r}\,,
\end{equation} 
under Assumption \ref{ass:disturb}, leads to
\begin{equation}
	\dot{V}_5 = -W_3 +  h_r\b{z}^\mathsf{T}_{\b{r}}\b{z}_{\bs{\Omega}}\,.
\end{equation}

The next step consists in ensuring that the angular velocity error $\b{z}_{\bs{\Omega}}$ goes to zero. To achieve this, a new candidate Lyapunov function is introduced:
\begin{equation} \label{eq:V6}
	V_6 \coloneq V_5 + \frac{h_\Omega}{2}\b{z}^\mathsf{T}_{\bs{\Omega}}\b{z}_{\bs{\Omega}}\,,
\end{equation}
where $h_\Omega \in \mathbb{R}^+$ is a positive auxiliary tuning parameter. The derivative of $V_6$ is given by
\begin{equation}
	\dot{V}_6 = -W_4 + \b{z}^\mathsf{T}_{\bs{\Omega}}\left(h_r\b{z_r} + k_{\Omega}\b{z}_{\bs{\Omega}} +h_\Omega\dot{\b{z}}_{\bs{\Omega}} \right)\,,
\end{equation}
where $k_\Omega \in \mathbb{R}^+$ is a positive scalar gain and $W_4 \coloneq W_1 +W_2 + W_3 + k_\Omega \b{z}^\mathsf{T}_{\bs{\Omega}} \b{z}_{\bs{\Omega}} \in \mathbb{R}^+_{0} $ is a non-negative auxiliary term. Looking at the angular velocity dynamics in (\ref{eq:dynamics_delta_r3}), we see that the angular acceleration $\dot{\bs{\Omega}}_{12}$ can be commanded as needed through the control torque $\bs{\Gamma}_c$. Let us then impose
\begin{equation}\label{eq:ang_accel_law}
	\bs{\Gamma}_c \coloneq -j_\perp\left(\frac{k_\Omega}{h_\Omega}\b{z}_{\bs{\Omega}} + \frac{h_r}{h_\Omega}\b{z_r} + \b{S}^2(\b{r}_3)\overline{\dot{\bs{\Omega}}}_{12c}\in T_{\b{r}_3}\mathbb{S}^2\right),
\end{equation}
where $\overline{\dot{\bs{\Omega}}}_{12c} \in \mathbb{R}^3$ (shown in Appendix \ref{app:aux_var}) is the known part of $\dot{\bs{\Omega}}_{12c} $, computed using a third estimate on $\b{b}$ defined as $\widehat{\b{b}}_3 \in \mathbb{R}^3$, such that
\begin{equation}\label{eq:om12cbardot}
	\overline{\dot{\bs{\Omega}}}_{12c} = \dot{\bs{\Omega}}_{12c} + \bs{\Theta}\, \widetilde{\b{b}}_3\,,
\end{equation}
with $\bs{\Theta} \in \mathbb{R}^{3\times3}$ being the matrix that gathers all dependencies on $\widetilde{\b{b}}_3$, also shown in Appendix \ref{app:aux_var}. Using the control law (\ref{eq:ang_accel_law}), together with the angular velocity error definition (\ref{eq:ang_vel_err}), the time derivative of (\ref{eq:V6}) can be rewritten as
\begin{equation}
	\dot{V}_6 = -W_4 + h_\Omega \b{z}^\mathsf{T}_{\bs{\Omega}}\,\bs{\Theta}\,\widetilde{\b{b}}_3\,.
\end{equation}
A final candidate Lyapunov function is now introduced with the goal of deriving an adaptation law for the third and final estimate $\widehat{\b{b}}_3$:
\begin{equation}\label{eq:V}
	V \coloneq V_6 +\frac{1}{2}\widetilde{\b{b}}_3^\mathsf{T}\bs{\Lambda}^{-1}_3\widetilde{\b{b}}_3\,,
\end{equation}
where $\bs{\Lambda}_3 \in \mathbb{R}^{3\times3}$ is a positive definite auxiliary gain matrix. Its time derivative is given by
\begin{equation}
	\dot{V} = -W_4 + \widetilde{\b{b}}^\mathsf{T}_3\left(\bs{\Lambda}_3^{-1}\dot{\widetilde{\b{b}}}_3 +h_\Omega\, \bs{\Theta}^\mathsf{T} \b{z}_{\bs{\Omega}} \right)\,.
\end{equation}
Therefore, the adaptation law
\begin{equation}\label{eq:adapt_3}
	\dot{\widehat{\b{b}}}_3 = -h_\Omega\bs{\Lambda}_3\,\bs{\Theta}^\mathsf{T}\b{z}_{\bs{\Omega}}\,,
\end{equation} 
under Assumption \ref{ass:disturb}, leads to
\begin{equation}\label{eq:V_derivative}
	\dot{V} = -W_4\leq0\,.
\end{equation}

At last, with all inputs of model (\ref{eq:dynamics_delta_r3}) already defined, the virtual thrust-vector input $\b{u}$ can be retrieved. From definition (\ref{eq:u3_bar}) and control law (\ref{eq:u3_bar_law}) we get
\begin{equation}
	u_3 \coloneq -m\left(\bs{\xi}^\mathsf{T}\b{r}_3 + \delta\,||\bs{\Omega}_{12}||^2\right)\,.
\end{equation}
Moreover, from (\ref{eq:Gamma}), together with (\ref{eq:delta_condition}), we have that
\begin{equation}
	-\b{S}^2(\b{e}_3)\b{u} = -\frac{1}{L} \b{S}(\b{e}_3)\b{R}^\mathsf{T}\bs{\Gamma}_c\,.
\end{equation}
Thus, the virtual thrust-vector input, which can be decomposed as
\begin{equation}
	\b{u} = \b{e}_3\b{e}^\mathsf{T}_3\b{u}-\b{S}^2(\b{e}_3)\b{u}\,,
\end{equation}
is given by
\begin{equation}\label{eq:control_law}
	\b{u} \coloneq -\frac{1}{L}\b{S}(\b{e}_3)\b{R}^\mathsf{T}\bs{\Gamma}_c-m\left(\bs{\xi}^\mathsf{T}\b{r}_3 + \delta\,||\bs{\Omega}_{12}||^2\right)\b{e}_3\,,
\end{equation} 
with $\bs{\Gamma}_c$ defined as in (\ref{eq:ang_accel_law}).

\begin{figure*}[t!]
	\centering
	\includegraphics[width=\textwidth]{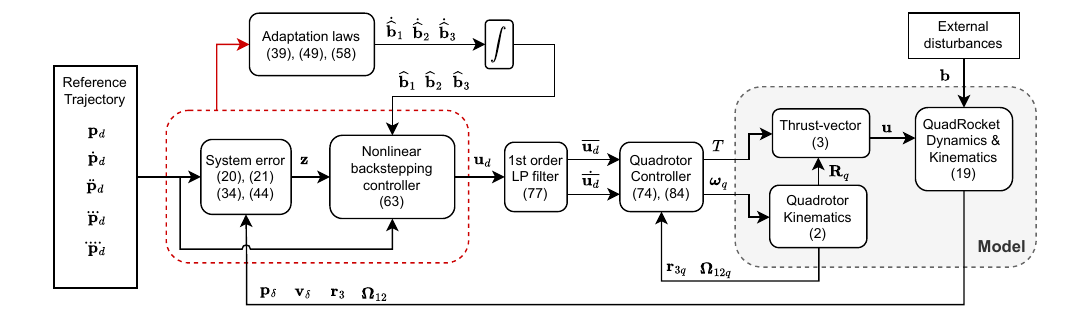}
	\caption{Diagram of control architecture.}
	\label{fig:control_diagram}
\end{figure*}

We are now ready to state the main theorem regarding the stability of the origin of the error system defined by $(\b{z}_1\,,\b{z}_2\,,\b{z_r}\,,\b{z}_{\bs{\Omega}})$, which dictates the trajectory tracking performance of the \textit{QuadRocket} under the single rigid body approach with idealized instantaneous thrust-vector actuation.
\begin{theorem}\label{th:main}
	Consider the dynamics (\ref{eq:dynamics}) and kinematics (\ref{eq:kinematics}) of the \textit{QuadRocket} when using the single rigid-body approach with idealized instantaneous thrust-vector actuation $\b{u}(t)$, and assume that condition (\ref{eq:delta_condition}) holds. Then, for a reference trajectory $\b{p}_d(t)$ of class at least $\mathcal{C}^4$, the control law (\ref{eq:control_law}), together with the adaptation laws (\ref{eq:adapt_1}), (\ref{eq:adapt2}), and (\ref{eq:adapt_3}), renders the tracking error variables $(\b{z}_1\,,\b{z}_2\,,\b{z_r}\,,\b{z}_{\bs{\Omega}})$ almost globally asymptotically convergent to zero in the presence of a constant bounded external disturbance, $\b{b}$. Moreover, all closed-loop signals remain bounded, and the disturbance estimate $\widehat{\mathbf{b}}_1$ converges to the true value $\b{b}$.    
\end{theorem}
\begin{proof}
 The Lyapunov function in (\ref{eq:V}) is positive definite and well-defined over the domain $\mathcal{O} \coloneq \{\b{z}\in \mathbb{R}^{12}:\b{r}_{3}\neq -\b{r}_{3d} \}$ as
	
	\begin{equation} \label{eq:V_compact}
		V \coloneq \frac{1}{2}\b{z}^\mathsf{T}\b{Q}\b{z} + \frac{1}{2}\widetilde{\b{b}}^\mathsf{T}\bs{\Lambda}^{-1}\widetilde{\b{b}}\,,
	\end{equation}
	where
	\begin{equation}
		\b{z} \coloneq \left[\:\b{z}^\mathsf{T}_1\:\:\b{e}^\mathsf{T}\:\:\b{z}^\mathsf{T}_\b{r}\:\:\b{z}^\mathsf{T}_{\bs{\Omega}}\: \right]^\mathsf{T} \in \mathbb{R}^{12}\,,
	\end{equation}
	\begin{equation}
		\widetilde{\b{b}} \coloneq  \left[\:\widetilde{\b{b}}^\mathsf{T}_1\:\:\widetilde{\b{b}}^\mathsf{T}_2\:\:\widetilde{\b{b}}^\mathsf{T}_3\: \right]^\mathsf{T} \in \mathbb{R}^{9}\,,
	\end{equation}
	\begin{equation}
		\b{Q} \coloneq \text{diag}\left(1,\,1,\,\frac{h_r}{1+\b{r}^\mathsf{T}_{3d}\b{r}_{3}},\,h_{\bs{\Omega}}\right) \bs{\otimes} \b{I}\in \mathbb{R}^{12\times12}
	\end{equation}
	is a positive definite matrix, and
	\begin{equation}
		\bs{\Lambda} \coloneq \bs{\Lambda}_1\bs{\oplus} \bs{\Lambda}_2\bs{\oplus} \bs{\Lambda}_3 \in \mathbb{R}^{9\times9}
	\end{equation}
	is also a positive definite matrix. When condition (\ref{eq:delta_condition}) holds, the dynamics and kinematics of the \textit{QuadRocket} with respect to the point $-\delta\b{e}_3$ in $\{B\}$, using the reduced attitude representation, are given by (\ref{eq:dynamics_delta_r3}). Moreover, we consider the ideal scenario where the quadrotor is able to instantaneously provide a desired force vector $\b{u}$(t) with respect to $\{B\}$. For these conditions, using (\ref{eq:V_derivative}), the time derivative of (\ref{eq:V_compact}) can be compactly written as
	\begin{equation}\label{eq:V_derivative2}
		\dot{V} = -\b{z}^\mathsf{T} \b{W} \b{z}\leq -\lambda_\text{min}\left(\b{W}\right)||\b{z}||^2\,,
	\end{equation}
	where 
	\begin{equation}
		\b{W} \coloneq  \b{K_p} \bs{\oplus} \b{K_v} \bs{\oplus} k_r\b{I} \bs{\oplus} k_\Omega\b{I} \in \mathbb{R}^{12\times12}
	\end{equation}
	is a positive definite matrix, and $\lambda_\text{min}\left(\b{W}\right)\in\mathbb{R}^+$ its smallest eigenvalue. Therefore, $\dot{V}$ is negative definite with respect to $\b{z}$, which implies that $V$ is bounded and, consequently, so are the errors $\b{z}$ and $\widetilde{\b{b}}$. Using this fact, one can prove that $\ddot{V}$ is also bounded (Appendix \ref{app:vbound}), allowing to infer that $\dot{V}$ is uniformly continuous. Hence, according to Barbalat's Lemma \cite[Lemma 4.2]{Slotine1991}, it is possible to conclude that $\dot{V} \rightarrow 0$, which implies that $\b{z}$ converges to zero. In turn, the convergence of $\b{z}$ to zero implies that $\b{z_1}\rightarrow \b{0}$, $\b{e}\rightarrow \b{0}$, $\b{z_2}\rightarrow \b{0}$, $\b{r}_3\rightarrow \pm\b{r}_{3d}$, and $\b{z}_{\bs{\Omega}}\rightarrow \b{0}$. The geometric ambiguity associated with the collinearity between $\b{r}_3$ and $\b{r}_{3d}$ is inevitable for static feedback controllers and is caused by the topological properties of $\mathbb{S}^2$ \cite{Bhat2000}. Nonetheless, the system configuration characterized by $\b{r}_3 = -\b{r}_{3d}$ forms a set that has measure zero. Therefore, under the weak assumption that initial conditions do not contain this set, which holds under all practical circumstances, the tracking error variables $(\b{z}_1\,,\b{z}_2\,,\b{z_r}\,,\b{z}_{\bs{\Omega}})$ almost globally asymptotically converge to the origin.
	
Finally, the convergence of $\b{z}_1$, $\b{z}_2$, and $\dot{\b{z}}_2$  to zero implies that 
		\begin{equation}
			\widehat{\b{b}}_1 \rightarrow \b{b} + \bs{\xi}+ \overline{u_3}\b{r}_3\,,
		\end{equation} 
		where $	\dot{\b{v}}_\delta = g\b{e}_3 + \overline{u_3}\,\b{r}_3 + \b{b}$, and where (\ref{eq:z2}), (\ref{eq:zeta}), (\ref{eq:xi}), and (\ref{eq:b1_tilde}) were used. Replacing $\overline{u_3}$ by the law imposed in (\ref{eq:u3_bar_law}), and using the fact that $\b{S}^2(\b{r}_3) = \b{r}_3\b{r}^\mathsf{T}_3 - \b{I}$, we get
		\begin{equation}\label{eq:b1_conv}
			\widehat{\b{b}}_1 \rightarrow \b{b} - \b{S}^2(\b{r}_3)\bs{\xi}\,.
		\end{equation} Using definitions (\ref{eq:r3d}) and (\ref{eq:attitude_err}), (\ref{eq:b1_conv}) can be rewritten as
		\begin{equation}
			\widehat{\b{b}}_1 \rightarrow \b{b} + ||\bs{\xi}||\b{S}(\b{r}_3)\b{z_r}\,.
		\end{equation} Finally, we note that $\b{z_r}\rightarrow \b{0}$, which implies that $\widehat{\b{b}}_1 \rightarrow \b{b}$, thus concluding our proof.
\end{proof}

\begin{remark}
	Convergence of the estimates $\widehat{\b{b}}_2$ and $\widehat{\b{b}}_3$ to the true value is dependent on the existence of persistent excitation from the reference trajectory, otherwise exhibiting a bounded estimation error. This is characteristic of adaptive backstepping design \cite{Krstic1995}.
\end{remark}
\color{black}
\begin{remark}\label{remark:theorem1}
	Theorem \ref{th:main} establishes stability results for the single rigid body model with thrust-vector actuation - a generic model typically used for rocket-like vehicles. In the context of the \textit{QuadRocket}, this model is used as a first approximation to the coupled dynamics of the system at hand, where the quadrotor is seen as an actuator which can instantaneously provide a thrust force at a given angle with respect to the main body. The validity of this approximation, which is used for control design purposes, was verified both in simulation and during flight testing.
\end{remark}

\subsection{Quadrotor Attitude Control}

With the virtual thrust-vector input already defined, we are left with the task of designing an appropriate control law for the angular velocity command given to the quadrotor, $\bs{\omega}_q$. The goal is that its attitude tracks the desired thrust direction, while providing the corresponding magnitude, thereby solving Problem \ref{prob:quadrotor_control}. Formally, we introduce a desired thrust-vector $\b{u}_d \in \mathbb{R}^3$, expressed in $\{B\}$, which corresponds to the virtual control input defined in (\ref{eq:control_law}). From the desired thrust-vector, the thrust input to the quadrotor immediately follows as
\begin{equation}\label{eq:thrust_input}
	T \coloneq ||\b{u}_d||\,.
\end{equation}
Then, we recall (\ref{eq:tvcinput}) and note that the actual force produced by the quadrotor, expressed in $\{B\}$, can be written as
\begin{equation}\label{eq:actual_u}
	\b{u} = -T\b{R}^\mathsf{T}\b{r}_{3q}\,,
\end{equation}
where $\b{r}_{3q} \coloneq \bs{\pi}_3\left(\b{R}_q\right)\in \mathbb{S}^2$ is the reduced attitude representation of the quadrotor. Expression (\ref{eq:actual_u}), together with definition (\ref{eq:thrust_input}), allows us to extract a desired direction for $\b{r}_{3q}$ as
\begin{equation} \label{eq:r3dq}
	\overline{\b{r}_{3dq}} \coloneq - \b{R}\frac{\overline{\b{u}_d}}{||\overline{\b{u}_d}||} \in \mathbb{S}^2\,,\hspace{10pt} \overline{\b{u}_d}\neq\b{0}\,,
\end{equation}
where $\overline{\b{u}_d}$ is the output of a first order filter applied to $\b{u}_d$, designed as
\begin{equation}\label{eq:filter}
	\dot{ \overline{\b{u}_d}}\,\tau_s + \overline{\b{u}_d} = \b{u}_d\,,\hspace{15pt}  \overline{\b{u}_d}(t_0) = \b{u}_d(t_0)\,,
\end{equation}
in which $\tau_s$ is the time constant and $t_0$ is the initial time. The use of the first order filter is borrowed from Dynamic Surface Control theory \cite{Swaroop2000} and will allow us to design a tracking controller without the need of computing the explicit time derivative of $\b{u}_d$. Let us then define a reduced attitude tracking error for the quadrotor as
\begin{equation}
	\b{e}_q \coloneq \b{S}(\overline{\b{r}_{3dq}})\b{r}_{3q} \in T_{\b{r}_{3q}}\mathbb{S}^2 \cap  T_{\b{r}_{3dq}}\mathbb{S}^2\,,
\end{equation}
and propose the candidate Lyapunov function
\begin{equation}\label{eq:Vq}
	V_q \coloneq 1 - \overline{\b{r}_{3dq}}^\mathsf{T}\b{r}_{3q}\,.
\end{equation}
Its time derivative can be written as
\begin{equation}
	\dot{V}_q  =  -W_q +\b{e}^\mathsf{T}_q\left(k_q\b{e}_q + \bs{\Omega}_{12q} - \bs{\Omega}_{12dq}\right)\,,
\end{equation}
where $W_q\coloneq k_q\b{e}^\mathsf{T}_q\b{e}_q \in \mathbb{R}^+_0$ is an auxiliary non-negative term, $k_q \in \mathbb{R}^+$ is a positive constant gain, $\bs{\Omega}_{12q} \coloneq \b{S}(\b{r}_{3q})\dot{\b{r}}_{3q} \in T_{\b{r}_{3q}}\mathbb{S}^2$, and  $\bs{\Omega}_{12dq} \coloneq \b{S}(\overline{\b{r}_{3dq}})\dot{\overline{\b{r}_{3dq}}} \in T_{\overline{\b{r}_{3dq}}}\mathbb{S}^2$. Thus, by setting
\begin{equation}\label{eq:om12qlaw}
	\bs{\Omega}_{12q} \coloneq -k_q\b{e}_q - \b{S}^2(\b{r}_{3q})\bs{\Omega}_{12dq}\,,
\end{equation}
the time derivative of $V_q$ becomes
\begin{equation}\label{eq:Vq_dot}
	\dot{V}_q = - W_q<0 \quad \forall\, \b{e}_q \neq \b{0}\,.
\end{equation}

Recalling (\ref{eq:om12}) and (\ref{eq:om}), we have that 
\begin{equation}
	\bs{\Omega}_{12q} = 	-\b{R}_q\b{S}^2(\b{e}_3)\,\bs{\omega}_q\,, 
\end{equation}
which means that $\bs{\Omega}_{12q}$ can be set through $\omega_{1q}$ and $\omega_{2q}$ -- the $x$ and $y$ components of $\bs{\omega}_q$. Thus, the remaining component, $\omega_{3q}$, can be used to drive $\omega_3$ to zero such that the second part of condition (\ref{eq:delta_condition}) is verified. The universal joint connecting both bodies imposes the restriction $\omega_3 = \omega_{3q}$, allowing us to simply set $\omega_{3q} = 0$. Therefore, the control law to govern the angular velocity of the quadrotor is set as
\begin{equation}\label{eq:omegaq_law}
	\bs{\omega}_q \coloneq \b{R}^\mathsf{T}_q\left(-k_q\b{e}_q - \b{S}^2(\b{r}_{3q})\bs{\Omega}_{12dq}\right)\,.
\end{equation}

\begin{proposition}\label{prop:quad}
	Consider the quadrotor attitude kinematics in (\ref{eq:quad_kinematics}), and that the quadrotor perfectly tracks the angular velocity command, $\bs{\omega}_q$. Then, the feedback law (\ref{eq:omegaq_law}) renders the point $\b{r}_{3q} = \overline{\b{r}_{3dq}}$ almost global asymptotically stable. Moreover, the actual thrust-vector input $\b{u}$ converges asymptotically to a ball $\mathcal{B} \coloneq \{\b{u}\in\mathbb{R}^3:||\b{u} - \b{u}_d||\leq r \}$ centered around $\b{u}_d$ of radius
	
	\begin{equation}
		r = \tau_s\,\mathop{\sup}\limits_{t \ge 0} \|\dot{	\b{u}}_d(t)\|\,.
	\end{equation}
\end{proposition}
\begin{proof}
The proof can be found in Appendix \ref{app:prop1proof}.
\end{proof}

\begin{remark}
Equation (\ref{eq:filter}) acts as a first-order low-pass filter on the desired thrust-vector, $\b{u}_d$, for which the cut-off frequency can be adjusted through the time constant $\tau_s$. As established in Proposition \ref{prop:quad}, the time constant will have an impact on the thrust-vector tracking error $||\b{u} - \b{u}_d||$. Hence, there exists a trade-off between minimizing the tracking error and limiting high-frequency noise.
\end{remark}

Finally, Fig. \ref{fig:control_diagram} presents an architecture diagram of the proposed integrated trajectory tracking solution, resulting from the combination of the virtual thrust-vector controller with the quadrotor control law.

\section{COMPUTER SIMULATION}

Prior to its implementation on the real vehicle, the overall control scheme for trajectory tracking was tested in a Matlab/Simulink simulation environment. The proposed solution was verified using the dynamics and kinematics model presented on Section \ref{sec:dyn_model}, corresponding to the numerical implementation of (\ref{eq:kinematics})-(\ref{eq:dynamics}). The built-in variable-step \textit{ode45} solver was used with a maximum allowed step size of 0.01 seconds\footnote{To ensure consistent numerical integration of \eqref{eq:r3_dot_eq}, a standard projection method \cite{Hairer2013} was used.}.

\subsection{Bump function-based reference trajectories}
To ensure that reference trajectories have the necessary smoothness, class $\mathcal{C}^\infty$ bump functions \cite{Tu2011} were used to directly define the time evolution of the desired inertial velocity. To that end, a smooth transition function is introduced: 
\begin{equation}
	f (s) \coloneq \begin{cases}
		0\,,\hspace{72.5pt}s\leq0\\
		\displaystyle\frac{e^{-\frac{1}{s}}}{e^{-\frac{1}{s}}+e^{-\frac{1}{1-s}}}\,,\hspace{20pt}0<s<1\\[10pt]
		1\,,\hspace{72.5pt}s\geq1		
	\end{cases}\,.
\end{equation}
Using the smooth transition function, $f(s)$, and defining $t_s$, $T_1$, and $T_2 \in \mathbb{R}$ as the start time, transition interval, and plateau interval, respectively, a generic bump function $\Psi:\mathbb{R}_0^+\rightarrow\mathbb{R}_0^+$ is constructed according to
\begin{equation}\label{eq:bump_function}
	\Psi (t\,|\,t_s,\, T_1,\, T_2) \coloneq 1-f\left(\textstyle\frac{(t-t_s-b)^2-a^2}{b^2-a^2}\right),
\end{equation}
where $a = T_2/2$ and $b = T_1 +a$. Figure \ref{fig:bump_function} displays an example bump function with the relevant time instants and intervals identified.
\begin{figure}[t]
	\centering
	\includegraphics{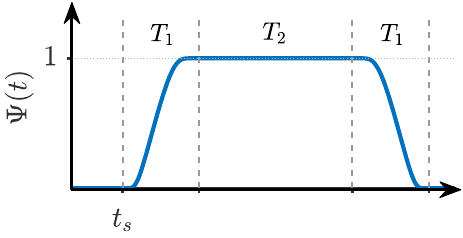} 
	\caption{Representation of the bump function.}
	\label{fig:bump_function}
\end{figure}

\subsection{Setup}

The reference trajectory, which was later used during flight arena testing, was derived from a desired inertial velocity profile (Fig. \ref{fig:ref_vel}) constructed using the generic bump function (\ref{eq:bump_function}). A conservative trajectory design was adopted to account for the prototype nature of the vehicle. The vehicle is required to perform an initial vertical climb, followed by a sideways motion, after which the vehicle simultaneously descends and moves towards the initial desired position. The exact inertial velocity profile is described by
\begin{equation}\label{eq:vel_profile}
	\b{\dot{p}}_d = \begin{bmatrix}
		0 \\[0.1cm]
		-0.5\Psi(t\,|11,2.9,0.2) + 	0.5\Psi(t\,|20,2.9,0.2)\\[0.1cm]
		-0.4\Psi(t\,|1,3.4,0.2) + 	0.28\Psi(t\,|20,4.9,0.2)
	\end{bmatrix},
\end{equation}
from which the desired inertial position can be integrated from a given initial condition, and the subsequent time derivatives can be explicitly derived.

For this simulated scenario, the initial position of the vehicle was set to $\b{p}(0) = \left[\,0.3\:\:-0.4\:\:-0.5\,\right]^\mathsf{T}\,$m, while the desired initial position of the control point was set to $\b{p}_d(0) = \left[\,0\:\:0\:\:-1\,\right]^\mathsf{T}\,$m. An initial inclination was induced according to $\b{R}(0) = \b{I} + \sin\theta\,\b{S}(\b{a}) +(1-\cos\theta)\b{S}^2(\b{a})$ with $\theta=3^\circ$ and $\b{a} = \left[\,-1\:\:-0.5\:\:0\,\right]^\mathsf{T}$. The quadrotor was initially aligned with the inertial frame, i.e. $\b{R}_q(0)=\b{I}$, leading to a misalignment with respect to the main body. As for the constant disturbance, its value was imposed as $\b{b} = \left[\,0.1\:\:-0.2\:\:-0.15\,\right]^\mathsf{T}\,$m/s$^2$. Additionally, a time delay of $0.02\,$s was introduced in the thrust and angular velocity commands to account for transmission delays and quadrotor response times. The system parameters are as follows: $m = 1.55\,$kg, $g=9.81\,$m.s$^{-2}$, $j_\perp =0.15\,$kg.m$^2$, and $L = 0.5\,$m. Table \ref{tab:sim_gains} collects the set of gains used in simulation.

\begin{table}[htb]
	\caption{Control gains used in simulation.}
	\label{tab:sim_gains}
	\centering
	\renewcommand{\arraystretch}{1.2} 
	\setlength{\tabcolsep}{6pt} 
	\begin{tabular}{>{\columncolor{gray!8}}c c}
		\rowcolor{gray!20}
		\textbf{Parameter} & \textbf{Value} \\
		\hline
		$\b{K_p},\,\b{K_v}$ & $2\,\b{I},\,1.5\,\b{I}$ \\
		$k_r,\,h_r$ & 60, 6 \\
		$k_\Omega,\,h_\Omega$ & 25, 15 \\
		$\bs{\Lambda}_1,\,\bs{\Lambda}_2,\,\bs{\Lambda}_3$ & $1\,\b{I},\,2\,\b{I},\,0.01\,\b{I}$ \\
		$k_q,\,\tau_s$ & 18, 0.03 \\
		\hline
	\end{tabular}
\end{table}

\subsection{Simulation Results}

\begin{figure}[t]
	\centering
	\includegraphics{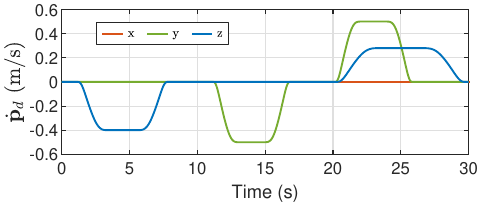} 
	\caption{Reference inertial velocity profile.}
	\label{fig:ref_vel}
\end{figure}

Figure \ref{fig:sim_traj} displays the 3-dimensional representation of the actual and reference trajectories. Starting with a significant position error and a counterproductive initial tilt, the \textit{QuadRocket} is able to converge to the desired trajectory and maintain accurate tracking throughout the flight.
\begin{figure}[t!]
	\centering
	\includegraphics{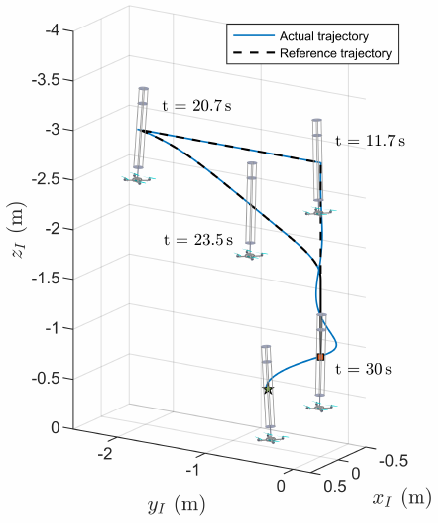} 
	\caption{Simulated trajectory tracking visualization with accurate representations of the vehicle at given points. The green star and red square represent the starting and final positions, respectively.}
	\label{fig:sim_traj}
\end{figure}
\begin{figure}[t]
	\centering
	\includegraphics[width=\columnwidth]{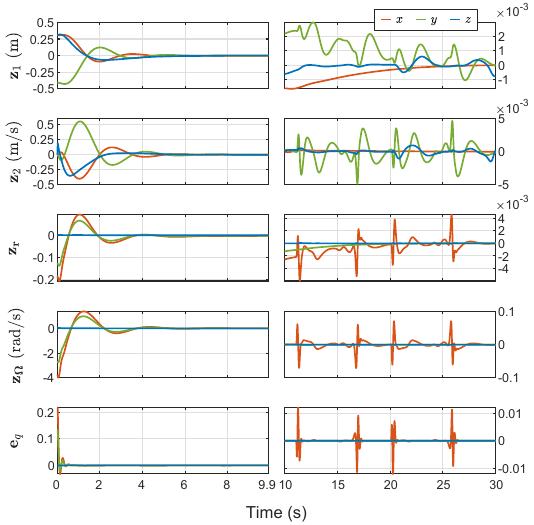} 
	\caption{Top to bottom: time evolution of $\b{z}_1$, $\b{z}_2$, $\b{z_r}$, $\b{z}_{\bs{\Omega}}$, and $\b{e}_q$. Left subplots: initial transient and convergence performance. Right subplots: steady-state performance.}
	\label{fig:sim_err}
\end{figure}

To quantitatively demonstrate the trajectory tracking performance, the time evolution of all errors is presented in Fig. \ref{fig:sim_err}. The left-side plots attest for the asymptotic convergence of position, velocity, and attitude errors to the origin during the first seconds of the flight. Even though asymptotic convergence was only proven for idealized thrust-vector actuation, the fast response of the quadrotor allows for the extension of this property to the interconnected system. Note that, appropriate tuning of the quadrotor attitude gain, $k_q$, leads to a comparatively faster convergence of the quadrotor attitude error, $\b{e}_q$. The right-side plots show the residual errors throughout the remainder of the flight. Noticeably, the step-like shape of the velocity bump function induces spikes in the attitude errors at the transition times, which likely propagate from the quadrotor attitude tracking error.

The \textit{QuadRocket} was able to track the desired trajectory in the presence of an unknown constant disturbance. Figure \ref{fig:sim_adapt} displays the time evolution of $\widehat{\b{b}}_1$, $||\widetilde{\b{b}}_1||$, $||\widetilde{\b{b}}_2||$, and $||\widetilde{\b{b}}_3||$. The first estimate on the disturbance, $\widehat{\b{b}}_1$, converges very close to the true value, as expected for the idealized model, while $\widehat{\b{b}}_2$ and $\widehat{\b{b}}_3$ exhibit steady state error.
\begin{figure}[t]
	\centering
	\includegraphics[width=\columnwidth]{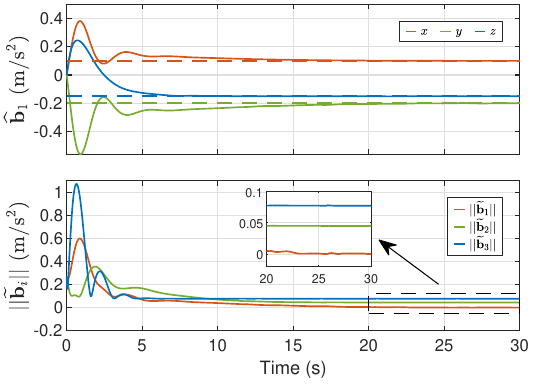} 
	\caption{Top: evolution of $\widehat{\b{b}}_1$. Bottom: evolution of $||\widetilde{\b{b}}_1||$, $||\widetilde{\b{b}}_2||$, and $||\widetilde{\b{b}}_3||$.}
	\label{fig:sim_adapt}
\end{figure}

Finally, the actuation signals are displayed in Fig. \ref{fig:sim_actuation}. The hover thrust magnitude stabilizes at a value below the weight due to the negative disturbance in the $z$-axis, a mark of the adaptive nature of the controller. As for the angular velocity command, after the initial stabilization effort, it is only excited at the transition intervals. The bottom plot shows the actual (provided by the quadrotor) and desired thrust-vector components orthogonal to $\b{r}_3$. The actual thrust-vector closely tracks the desired values,
\begin{figure}[t]
	\centering
	\includegraphics[width=\columnwidth]{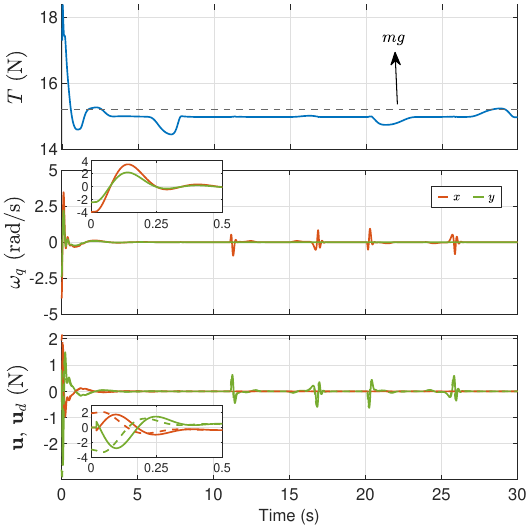} 
	\caption{Actuation signals. Top to bottom: Thrust magnitude, quadrotor angular velocity command, and actual vs desired thrust-vector components orthogonal to $\b{r}_3$.}
	\label{fig:sim_actuation}
\end{figure}
once again indicating that idealizing the quadrotor as a thrust-vector actuator is not only a reasonable assumption but also an effective one. The zoomed section in the first 0.5$\,$s attests for the fast convergence, limited by the imposed 0.02$\,$s delay. Overall, simulation results provide an empirical confirmation of the closed-loop stability of the interconnected system, as stated in Remark \ref{remark:theorem1}.
\section{FLIGHT EXPERIMENT AND RESULTS}

The viability of the proposed control scheme in computer simulation paves the way for the next logical step: to test it on the actual vehicle in a controlled environment - an indoor flight arena equipped with a motion capture system.

\subsection{Setup}

\begin{figure}[t]
	\centering
	\includegraphics[width=\columnwidth]{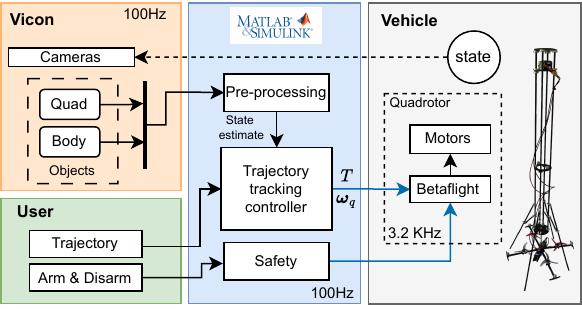} 
	\caption{Experimental setup. Blue arrows indicate RC communication links.}
	\label{fig:experimental_setup}
\end{figure}
\begin{figure*}[t!]
	\centering
	\includegraphics[width=\textwidth]{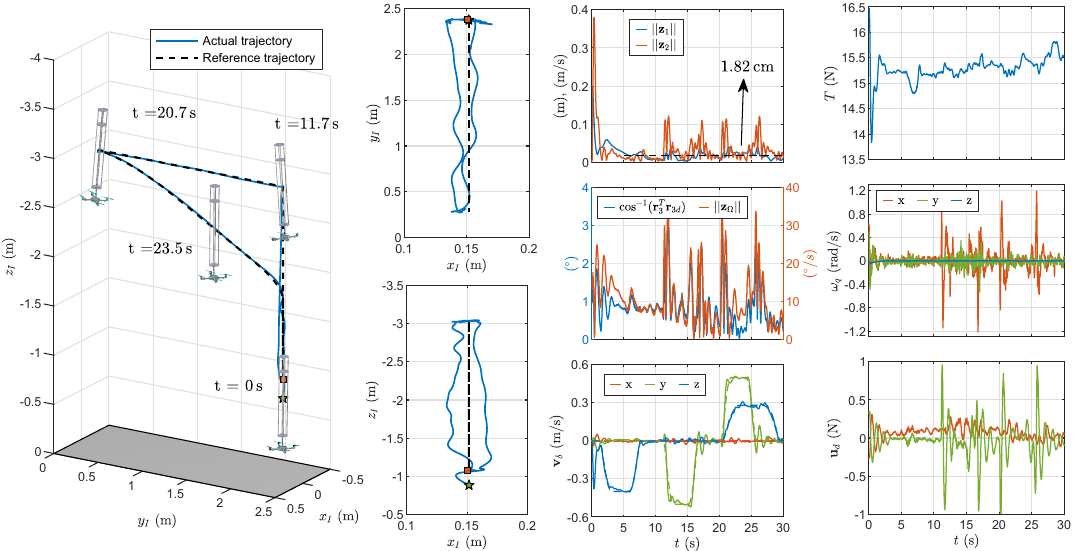} 
	\caption{Experimental trajectory tracking results. On the left, actual and reference trajectories with snapshots of the vehicle. Center-left column displays two planar views. Center-right column displays the evolution (top to bottom) of i) the norm of position and velocity errors, ii) the angle between $\b{r}_3$ and $\b{r}_{3d}$ and the norm of the angular velocity error, and iii) desired and actual inertial velocity. On the right (top to bottom), evolution of the thrust magnitude, quadrotor angular velocity command, 
    and components of the desired thrust-vector orthogonal to $\b{r}_3$.}
	\label{fig:experimental_results}
\end{figure*}

The experimental setup developed at the SCORE Laboratory, University of Macau, is implemented in a MATLAB/Simulink environment that enables seamless integration of sensors, control algorithms, and vehicle communication systems. Figure \ref{fig:experimental_setup} schematizes the experimental setup, describing the information flow between each subsystem. The states of the \textit{QuadRocket} are measured using a VICON motion capture system comprising a total of 30 Vantage cameras arranged in an 11 × 5 × 7 m volume. This high-performance system is able to operate with sub-millimeter accuracy at 100$\,$Hz, which allows for an accurate tracking of the reflective markers attached to the main body and quadrotor. Position and attitude measurements of each object are received in Simulink via Ethernet connection and pre-processed to recover the full state of the \textit{QuadRocket} - $\b{p}_\delta$, $\b{v}_\delta$, $\b{r}_3$, and $\bs{\Omega}_{12}$. We stress that the VICON motion capture system is a measurement tool, not a control crutch, meaning that it is used as a high-precision ground-truth measurement tool, not as a substitute for onboard sensing in the control loop. Moreover, the angular velocity of the rigid-bodies is obtained implicitly from a first-order Euler discrete differentiation of their rotation matrices, which means that the controller receives a state ``estimate" at a fixed rate with associated noise and latency, mirroring a real sensor suite.

The trajectory tracking controller is implemented in Simulink and runs at a frequency of 100$\,$Hz, using the state estimate and the user-defined trajectory to compute the inputs to the quadrotor, i.e., $T$ and $\bs{\omega}_q$. These inputs are sent to the quadrotor through an RC communication link at 100$\,$Hz and interpreted by the onboard flight computer running the Betaflight software. The onboard flight computer adjusts the current sent to each motor at the frequency of 3.2$\,$kHz so as to track the commanded thrust and angular velocity. The RC link contains an additional arm/disarm signal which is sent by the user from Simulink to start or abort the mission. The abort signal can also be automatically triggered if given state-dependent conditions are verified, such as user-defined attitude and/or position limits. The desired trajectory for the test flight was derived from the inertial velocity profile defined in (\ref{eq:vel_profile}). However, the desired initial position was set to $\b{p}_d(0) = \b{p}_\delta(0) - 0.2\,\b{e}_3\,$m so that the desired trajectory would adjust to the initial position (as placed in the arena) and that the vehicle would rapidly gain some elevation with respect to the floor. Apart from the adaptation gains, all other parameters were kept equal to the simulation values (Tab. \ref{tab:sim_gains}). More conservative values were used for the adaptation gains, prioritizing the first adaptation layer: $\b{\Lambda}_1 = 0.5\,\b{I}$, $\b{\Lambda}_2 =0.02\,\b{I}$, and $\b{\Lambda}_3 =0.001\,\b{I}$.

\subsection{Experimental Results}

The main trajectory tracking results of the flight experiment conducted at the indoor flight arena are shown in Fig. \ref{fig:experimental_results}. Additionally, a video of this same experiment is available at: \url{https://youtu.be/V5t_TqCEJUo}. The \textit{QuadRocket} was able to track the desired trajectory with centimeter-level precision and moderate actuation effort. In fact, after an initial transient of approximately 5$\,$s, the mean value of $||\b{z}_1||$ is calculated at 1.82$\,$cm. As verified in simulation, the transient intervals of the reference velocity bumps induce spikes which propagate from the quadrotor tracking error (Fig. \ref{fig:quad_tracking}) up to the position error. Most noticeably, the linear and angular velocity errors, $\b{z}_2$ and $\b{z}_{\bs{\Omega}}$, reach maximum values during these periods. The impact of high frequency noise coming from linear and angular velocity estimates is also visible, especially in the angular velocity command - a marked difference between the noiseless simulated results and the experimental ones.  When the vehicle starts or stops moving, the actuation effort is increased as its underactuated nature requires a tilt towards or against the direction of motion, respectively. Moreover, by analyzing the time evolution of $\b{v}_\delta$ during these instants, it is possible to infer that non-minimum phase behavior, although minimized, is not completely canceled. This fact indicates that some error was present in the values of the physical parameters required to compute $\delta$. Indeed, the accurate tracking results were obtained in the inevitable presence of model uncertainty, as the one posed by the diagonal inertia matrix assumption, attesting for the robustness of the proposed tracking law and for the validity of the assumptions used to simplify control design.

The high precision tracking with extremely reduced steady-state errors was enabled by the adaptive nature of the proposed control scheme. The time evolution of $\widehat{\b{b}}_1$, $\widehat{\b{b}}_2$, and  $\widehat{\b{b}}_3$ is shown in Fig. \ref{fig:adapt_exp}. 
\begin{figure}[htb]
	\centering
	\includegraphics[width=\columnwidth]{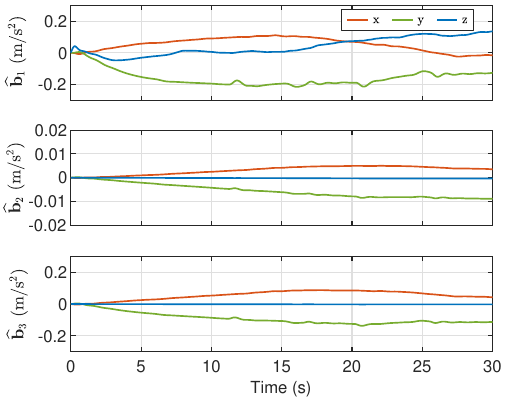} 
	\caption{Time evolution of $\widehat{\b{b}}_1$, $\widehat{\b{b}}_2$, and  $\widehat{\b{b}}_3$.}
	\label{fig:adapt_exp}
\end{figure}
A considerable disturbance value was estimated by the first adaptation layer, $\widehat{\b{b}}_1$, with emphasis on the $y$-axis, which is attributed to a possible misalignment between the quadrotor attachment point and the center of mass of the vehicle. The increase on the $z$-axis value towards the end of the flight reflects the thrust performance loss with the decrease of the battery voltage. In fact, the variation of the estimated values throughout the flight indicates the presence of time-varying and state-dependent disturbances, which the controller was able to counteract even though theoretical guarantees are built on the constant disturbance assumption. Additionally, a comparatively lower magnitude of $\widehat{\b{b}}_2$ is verified, which indicates that the corresponding gain was overly conservative.

Finally, and as predicted by simulation, the quadrotor attitude tracking is sufficiently fast so that it can be considered as a thrust-vector actuator to the main body, thus confirming the applicability of this platform as a rocket prototype. The top plot in Fig. \ref{fig:quad_tracking} demonstrates this fact by showcasing the fast and accurate tracking of $\overline{\b{r}_{3dq}}$ by $\b{r}_{3dq}$. This property allows for the extension of closed-loop stability to the interconnected system, validating the single rigid-body assumption as an effective one for control design purposes. The bottom plot compares the angular displacement between $\b{r}_{3q}$ and $\b{r}_{3qd}$ - a measure of the actual tracking error - against the angular displacement between $\b{r}_{3q}$ and $\overline{\b{r}_{3dq}}$ - a measure of the error used in dynamic surface control. The actual error has a larger steady-state value and the tracking performance is significantly worse during transients, highlighting the downfalls of relying on dynamic surface control to avoid explicit differentiation of the desired thrust-vector input, $\b{u}_d$. Nonetheless, the use of this control strategy for quadrotor attitude tracking did not compromise the overall trajectory tracking performance and simplified the resulting control scheme.

\begin{figure}[t]
	\centering
	\includegraphics[width=\columnwidth]{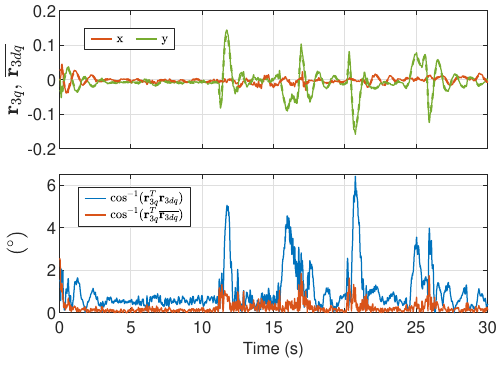} 
	\caption{Quadrotor attitude tracking performance. Top: $\b{r}_{3q}$ vs $\overline{\b{r}_{3dq}}$ - $x$ and $y$ components. Bottom: Angle between $\b{r}_{3q}$ and i) $\overline{\b{r}_{3dq}}$, and ii) $\b{r}_{3dq}$.    }
	\label{fig:quad_tracking}
\end{figure}

\section{CONCLUSION}
This paper introduced \textit{QuadRocket}, a quadrotor-based rocket prototype that serves as a low-cost, low-risk aerial robotic testbed for thrust-vector control of rocket-like vehicles. The coupled quadrotor–cylindrical-body system is modeled as a single axisymmetric rigid body actuated by a vectored force, using a reduced-attitude representation on the two-sphere to explicitly exploit axial symmetry and decouple yaw from thrust-vector direction. On this model, we derived an adaptive backstepping controller that provides almost-global trajectory tracking in the presence of unknown constant disturbances, while a control-point transformation mitigates non-minimum-phase behavior induced by the underactuated dynamics.
To bridge the gap between the virtual thrust-vector input and the real actuation, the quadrotor is treated as a thrust-vector actuator. A dynamic-surface-based attitude controller is then designed to track the desired thrust-vector while accounting for actuation dynamics and avoiding explicit differentiation of virtual control signals. The complete control architecture was first validated in simulation and then experimentally demonstrated in an indoor motion-capture arena, where the \textit{QuadRocket} achieved accurate trajectory tracking and effective disturbance compensation, despite actuation delays and sensing noise.
Overall, the results confirm that a compact quadrotor platform can reliably emulate thrust-vector-controlled rocket dynamics and that the proposed control framework is suitable for experimental validation of advanced guidance and control strategies. Future work will explore more aggressive maneuvers in the presence of external disturbances and parameter variations, integration with onboard state estimation and sensing, and extension of the framework to multi-vehicle cooperative scenarios and other thrust-vector-controlled robotic systems.\vspace*{-.5pc}

\section*{APPENDIX}

\subsection{Auxiliary Computations}\label{app:aux_var}

For completeness and ease of implementation, this first appendix provides all remaining essential formulas required to reproduce the proposed methodology. Intermediate steps and derivations are omitted to maintain conciseness.

\subsubsection{Computation of $\widehat{\bs{\Omega}}_{12d}$}

\begin{subequations}
	\begin{align}
		&\widehat{\bs{\Omega}}_{12d} \coloneq \b{S}(\b{r}_{3d})\,\widehat{\dot{\b{r}}}_{3d}\\
		&\widehat{\dot{\b{r}}}_{3d} \coloneq -||\bs{\xi}||^{-1}\b{S}^2(\b{r}_{3d})\,\widehat{\dot{\bs{\xi}}}\\
		&\widehat{\dot{\bs{\xi}}} \coloneq \widehat{\dot{\bs{\zeta}}} + \dot{\widehat{\b{b}}}_1 - \dddot{\b{p}}_d\\
		&\widehat{\dot{\bs{\zeta}}} \coloneq \left(\b{I} + \b{K_v}\b{K_p} \right)\b{z}_2 + \left(\b{K_p} + \b{K_v} \right)\widehat{\dot{\b{z}}}_2\\
		&\widehat{\dot{\b{z}}}_2 \coloneq g\b{e}_3 + \overline{u_3}\,\b{r}_3 + \widehat{\b{b}}_2 - \ddot{\b{p}}_d
	\end{align}
\end{subequations}

\subsubsection{Computation of $\overline{\dot{\bs{\Omega}}}_{12c}$}

\begin{multline}\label{eq:om12cbardot2}
	\overline{\dot{\bs{\Omega}}}_{12c} \coloneq   \frac{1}{h_r}\left[\b{S}(\b{r}_3)\b{e}\,\b{r}_{3d}^\mathsf{T}\,\overline{\dot{\bs\xi}}+||\bs{\xi}||\left(\b{S}(\dot{\b{r}}_3)\b{e+\b{S}(\b{r}_3)\overline{\dot{\b{e}}}}\right)\right]\\ -\frac{k_r}{h_r}\overline{\dot{\b{z}}}_\b{r} - \left(\b{S}(\dot{\b{r}}_3)\b{S}(\b{r}_3)+\b{S}(\b{r}_3)\b{S}(\dot{\b{r}}_3)\right)\widehat{\bs{\Omega}}_{12d}\\- \b{S}^2(\b{r}_3)\overline{\dot{\bs{\Omega}}}_{12d}\
\end{multline}
where
\begin{subequations}
	\begin{align}
		&\overline{\dot{\b{z}}}_\b{r}\coloneq \b{S}(\b{r}_{3d})\dot{\b{r}}_3-\b{S}(\b{r}_{3})\overline{\dot{\b{r}}}_{3d}\\
		&\overline{\dot{\b{r}}}_{3d} \coloneq -||\bs{\xi}||^{-1}\b{S}^2(\b{r}_{3d})\,\overline{\dot{\bs{\xi}}}\\
		&\overline{\dot{\bs{\xi}}} \coloneq \overline{\dot{\bs{\zeta}}} + \dot{\widehat{\b{b}}}_1 - \dddot{\b{p}}_d\\
		&\overline{\dot{\bs{\zeta}}} \coloneq \left(\b{I} + \b{K_v}\b{K_p} \right)\b{z}_2 + \left(\b{K_p} + \b{K_v} \right)\overline{\dot{\b{z}}}_2\\
		&\overline{\dot{\b{e}}} \coloneq \b{K_p}\,\b{z}_2 + 	\overline{\dot{\b{z}}}_2\\
		&\overline{\dot{\b{z}}}_2 \coloneq 	\widehat{\dot{\b{z}}}_2 - \widehat{\b{b}}_2 + \widehat{\b{b}}_3\\
		&\overline{\dot{\bs{\Omega}}}_{12d}\coloneq -\b{S}(\widehat{\dot{\b{r}}}_{3d})\overline{\dot{\b{r}}}_{3d} + \b{S}(\b{r}_{3d})\overline{\ddot{\b{r}}}_{3d}
	\end{align}
\end{subequations}
and
\begin{multline}
	\overline{\ddot{\b{r}}}_{3d} \coloneq -||\bs{\xi}||^{-1}\b{S}^2(\b{r}_{3d})\,\overline{\ddot{\bs{\xi}}}+ \left(\frac{\b{r}_{3d}^\mathsf{T}\,\overline{\dot{\bs\xi}}}{||\bs{\xi}||^2}\right)\b{S}^2(\b{r}_{3d})	\widehat{\dot{\bs{\xi}}} \\-||\bs{\xi}||^{-1} \left(\b{S}(\overline{\dot{\b{r}}}_3)\b{S}(\b{r}_3)+\b{S}(\b{r}_3)\b{S}(\overline{\dot{\b{r}}}_3) \right)	\widehat{\dot{\bs{\xi}}}
\end{multline}
where
\begin{subequations}\label{eq:om12cbardot_terms}
	\begin{align}
		&\overline{\ddot{\bs{\xi}}} \coloneq \overline{\ddot{\bs{\zeta}}} + \bs\Lambda_1\overline{\dot{\b{e}}} - \ddddot{\b{p}}_d\\
		&\overline{\ddot{\bs{\zeta}}} \coloneq \left(\b{I} + \b{K_v}\b{K_p} \right)\overline{\dot{\b{z}}}_2 + \left(\b{K_p} + \b{K_v} \right)\overline{\ddot{\b{z}}}_2\\
		&\overline{\ddot{\b{z}}}_2 \coloneq \overline{u_3}\,\dot{\b{r}}_3 + \overline{ \dot{\overline{u_3}}}\,\b{r}_3 + \dot{\widehat{\b{b}}}_2 - \dddot{\b{p}}_d\\
		&\overline{ \dot{\overline{u_3}}} \coloneq - \b{r}_3^\mathsf{T}\,\overline{\dot{\bs{\xi}}} - \bs{\xi}^\mathsf{T}\dot{\b{r}}_3
	\end{align}
\end{subequations}

\subsubsection{Computation of $\bs{\Theta}$}
Let us define $\bs\Theta_1$, $\bs\Theta_2$, $\bs\Theta_3$ and $\bs\Theta_4\in\mathbb{R}^{3\times3}$ such that
\begin{subequations}
	\begin{align}
		&\frac{d}{dt}\,\widehat{\dot{\bs{\xi}}} = \overline{\ddot{\bs{\xi}}} + \bs\Theta_1\,\widetilde{\b{b}}_3\,,\\
		&\dot{\b{r}}_{3d} = \overline{\dot{\b{r}}}_{3d} + \bs\Theta_2\,\widetilde{\b{b}}_3\,,\\
		&\frac{d}{dt}\,\widehat{\dot{\b{r}}}_{3d} = 	\overline{\ddot{\b{r}}}_{3d} + \bs\Theta_3\,\widetilde{\b{b}}_3\,,\\
		&\frac{d}{dt}\,	\widehat{\bs{\Omega}}_{12d} = 	\overline{\dot{\bs{\Omega}}}_{12d} + \bs\Theta_4\,\widetilde{\b{b}}_3\,.
	\end{align}
\end{subequations}
Then, we have that
\begin{subequations}
	\begin{align}
		\bs\Theta_1 &\coloneq 
		\begin{aligned}[t]
			&-\bs\Lambda_1 - \b{I} - \b{K_v}\b{K_p} \\
			&+ (\b{K_p} + \b{K_v})\, \b{r}_3\b{r}_3^\mathsf{T} (\b{K_p} + \b{K_v})
		\end{aligned} \\[3pt]
		\bs\Theta_2 &\coloneq 
		\begin{aligned}[t]
			&-\|\bs{\xi}\|^{-1}\b{S}^2(\b{r}_{3d})(\b{K_p} + \b{K_v})
		\end{aligned} \\[3pt]
		\bs\Theta_3 &\coloneq
		\begin{aligned}[t]
			&\|\bs\xi\|^{-1}\left(2\,\widehat{\dot{\bs{\xi}}}\,\b{r}_{3d}^\mathsf{T} -\b{r}_{3d}\,\widehat{\dot{\bs{\xi}}}^\mathsf{T} - (\b{r}_{3d}^\mathsf{T}\,\widehat{\dot{\bs{\xi}}})\b{I}\right)\bs\Theta_2\\
			&- \|\bs\xi\|^{-1}\b{S}^2(\b{r}_{3d})\bs\Theta_1 
			- \frac{\b{S}^2(\b{r}_{3d})}{\|\bs\xi\|^2}\, \widehat{\dot{\bs{\xi}}}\,\b{r}_{3d}^\mathsf{T}\, (\b{K_p} + \b{K_v})
		\end{aligned} \\[3pt]
		\bs\Theta_4 &\coloneq
		\begin{aligned}[t]
			&\b{S}(\b{r}_{3d})\bs\Theta_3 - \b{S}(\widehat{\dot{\b{r}}}_{3d})\bs\Theta_2
		\end{aligned}
	\end{align}
\end{subequations}
and, finally, 
\begin{equation}
	\bs\Theta \coloneq
	\begin{aligned}[t]
		&-\frac{k_r}{h_r}\b{S}(\b{r}_3)\bs\Theta_2 + \b{S}^2(\b{r}_3)\bs\Theta_4  + \frac{||\bs\xi||}{h_r}\b{S}(\b{r}_3) \\
		&+ \frac{1}{h_r}\b{S}(\b{r}_3)\b{e}\,\b{r}_{3d}^\mathsf{T}(\b{K_p} + \b{K_v})\,.
	\end{aligned}
\end{equation}

\subsection{Boundedness of $\ddot{V}$}\label{app:vbound}
The second time derivative of $V$ can be written as 
\begin{equation*}
	\ddot{V} = -2\,\b{z}^\mathsf{T}_1\b{K_p}\dot{\b{z}}_1 - 2\,\b{e}^\mathsf{T}\b{K_v}\dot{\b{e}} -2\,k_r\b{z}^\mathsf{T}_{\b{r}}\dot{\b{z}}_{\b{r}} -2\,k_\Omega\b{z}^\mathsf{T}_{\bs{\Omega}}\dot{\b{z}}_{\bs{\Omega}}\,.
\end{equation*}
Since we have already established that both the tracking and estimation errors are bounded, we must now check for the boundedness of their time derivatives.

Firstly, recall that the reference curves are smooth and have bounded time derivatives. Moreover, in the presence of bounded disturbances, the fact that $\widetilde{\b{b}}$ is bounded, implies the boundedness of $\widehat{\b{b}}_1$, $\widehat{\b{b}}_2$, and $\widehat{\b{b}}_3$. The fact that both $\b{z_1}$ and $\b{e}$ are bounded implies, by definition, the boundedness of $\b{z}_2$, which in turn implies the boundedness of $\dot{\b{z}}_1$ (note that $\dot{\b{z}}_1 = \b{z}_2$). The velocity error dynamics, given by $\dot{\b{z}}_2 = \dot{\b{v}}_\delta - \ddot{\b{p}}_d$, will also be bounded, owing to the fact that the input $\overline{u_3}$, defined in (\ref{eq:u3_bar_law}), is bounded. The boundedness of $\overline{u_3}$ follows from the boundedness of $\bs{\xi}$ in (\ref{eq:xi}), which depends only on terms already verified to be bounded. Immediately, note that $\dot{\b{e}} = \dot{\b{z}}_2 + \b{K_p}\dot{\b{z}}_1$ is bounded. 

Consider now the following time derivatives: $\dot{\bs{\zeta}} = \left(\b{I} + \b{K_v}\b{K_p} \right)\dot{\b{z}}_1 + \left(\b{K_p} + \b{K_v} \right)\dot{\b{z}}_2$, $\dot{\bs{\xi}} = \dot{\bs{\zeta}} + \dot{\widehat{\b{b}}}_1 - \dddot{\b{p}}_d$, $\dot{\b{z}}_{\b{r}} = \b{S}(\dot{\b{r}}_{3d})\b{r}_3 + \b{S}(\b{r}_{3d})\dot{\b{r}}_{3}$, $\dot{\b{r}}_{3d} = - ||\bs{\xi}||^{-1}\b{S}^2(\b{r}_{3d})\dot{\bs{\xi}}$, and $\dot{\b{r}}_3 = - \b{S}(\b{r}_3)\bs{\Omega}_{12}$. Then, we have that $\dot{\bs{\zeta}}$ is bounded, which together with the boundedness of $\dot{\widehat{\b{b}}}_1$ (see (\ref{eq:adapt_1})), implies that $\bs{\dot{\xi}}$ is bounded. In addition, both $\b{r}_3$ and $\b{r}_{3d}$ are unit vectors on the 2-sphere, thus being bounded. Therefore, $\dot{\b{r}}_{3d}$ is also bounded. As for $\dot{\b{r}}_{3}$, it follows from the boundedness of $\b{z}_{\bs{\Omega}}$ that a bounded angular velocity command $\bs{\Omega}_{12c}$ leads to a bounded signal. From (\ref{eq:om12c}) and (\ref{eq:om12dhat}) we have that $\bs{\Omega}_{12c}$ is in fact bounded as it depends on bounded terms only, proving the boundedness of $\dot{\b{r}}_{3}$. It is then possible to conclude that $\dot{\b{z}}_{\b{r}}$ is bounded.

At last, we have left to prove the boundedness of $\dot{\b{z}}_{\bs{\Omega}}$. Consider the following time derivatives: $\ddot{\b{z}}_1 = \dot{\b{z}}_2$, $\ddot{\b{z}}_2 = \dot{\overline{u_3}}\b{r}_3 + \overline{u_3}\dot{\b{r}}_3 - \ddddot{\b{p}}_d$, $\dot{\overline{u_3}} = -\dot{\bs{\xi}}^\mathsf{T}\b{r}_3 -\bs{\xi}^\mathsf{T}\dot{\b{r}}_3 $, $\ddot{\bs{\zeta}} = \left(\b{I} + \b{K_v}\b{K_p} \right)\ddot{\b{z}}_1 + \left(\b{K_p} + \b{K_v} \right)\ddot{\b{z}}_2$, $\ddot{\bs{\xi}} = \ddot{\bs{\zeta}} + \ddot{\widehat{\b{b}}}_1 - \dddot{\b{p}}_d$, and $\ddot{\widehat{\b{b}}}_1 = \bs{\Lambda}_1\dot{e}$, $\ddot{\b{r}}_{3d} = -\frac{d}{dt}\left(||\bs{\xi}||^{-1}\b{S}^2(\b{r}_{3d})\dot{\bs{\xi}}\right)$. It follows that $\ddot{\bs{\xi}}$ is bounded, which in turn confirms the boundedness of $\ddot{\b{r}}_{3d}$. Thus, $\dot{\bs{\Omega}}_{12d} $ is also bounded. From (\ref{eq:adapt2}) and (\ref{eq:Xi}) we have that $ \dot{\widetilde{\b{b}}}_2$ is bounded, allowing us to conclude the boundedness of $\dot{\widehat{\bs{\Omega}}}_{12d}$ and, consequently, of $\dot{\bs{\Omega}}_{12c}$ (see (\ref{eq:om12dhat}) and (\ref{eq:om12c})). Therefore, since $\dot{\b{z}}_{\bs{\Omega}} = \dot{\bs{\Omega}}_{12} - \dot{\bs{\Omega}}_{12c}$, if $\dot{\bs{\Omega}}_{12}$ is bounded, then $\dot{\b{z}}_{\bs{\Omega}} $ will also be bounded. The boundedness of $\dot{\bs{\Omega}}_{12}$  is determined by the input $\bs{\Gamma}_c$ in (\ref{eq:ang_accel_law}), where $	\overline{\dot{\bs{\Omega}}}_{12c} $ is the only signal whose boundedness is left to prove. However, we note from (\ref{eq:om12cbardot2})-\eqref{eq:om12cbardot_terms} that $\overline{\dot{\bs{\Omega}}}_{12c} $ is bounded, thus concluding our proof.

\subsection{Proof of Proposition \ref{prop:quad}}\label{app:prop1proof}

Lyapunov function $V_q$, defined in (\ref{eq:Vq}), is positive semi-definite and its codomain is the compact set  $\mathcal{C} = [0,\,2]$. The limits of the codomain exclusively occur at the points where $\b{e}_q$ goes to zero, i.e, $V_q(\b{r}_{3q}=\overline{\b{r}_{3dq}}) = 0$ and $V_q(\b{r}_{3q}=-\overline{\b{r}_{3dq}})=2$. Moreover, its time derivative (\ref{eq:Vq_dot}) is negative definite about $\b{e}_q$, vanishing when $\b{r}_{3q}=\pm\overline{\b{r}_{3dq}}$. Therefore, for any initial condition in the set $\mathcal{I} = \{\b{r}_{3q} \in \mathbb{S}^2: \b{r}_{3q}\neq-\overline{\b{r}_{3dq}}\}$, $\b{r}_{3q}$ will asymptotically converge to $\overline{\b{r}_{3dq}}$, thus proving the first part of Proposition \ref{prop:quad}. By definition, the almost global asymptotic convergence of $\b{r}_{3q}$ to $\overline{\b{r}_{3dq}}$ directly implies the almost global convergence of $\b{u}$ to $\overline{\b{u}_d}$ (see (\ref{eq:actual_u}) and (\ref{eq:r3dq})). Let us now define the filtering error as $\b{e}_f = \overline{\b{u}_d} - \b{u}_d$ arising from the first-order low-pass filter (\ref{eq:filter}). Consequently, we have that $\b{u}\rightarrow \b{u}_d + \b{e}_f$, with the filtering error dynamics being given by $\dot{\b{e}}_f = -(1/\tau_s)\b{e}_f - \dot{\b{u}}_d$. For a norm bounded time derivative of the desired thrust-vector, i.e,
\begin{equation}
	||\dot{\b{u}}_d(t)||\leq \dot{u}_{d_{max}} = \mathop{\sup}\limits_{t \ge 0} \|\dot{	\b{u}}_d(t)\|\,,
\end{equation}
it follows that
\begin{equation}
	||\b{e}_f(t)|| \leq \dot{u}_{d_{max}}\,\tau_s\,\left(1-e^{-(1/\tau_s)t}\right)\,.
\end{equation}
Therefore, as $t\rightarrow\infty$, the norm of the filtering error will be bounded by $\tau_s\,\dot{u}_{d_{max}}$, implying that $\b{u}$ asymptotically converges to the ball of radius $\tau_s\,\dot{u}_{d_{max}}$ centered around $\b{u}_d$, thus concluding our proof.

\bibliographystyle{IEEEtran}

\bibliography{refs}

\begin{IEEEbiography}[{\includegraphics[width=1in,height=1.25in,clip,keepaspectratio]{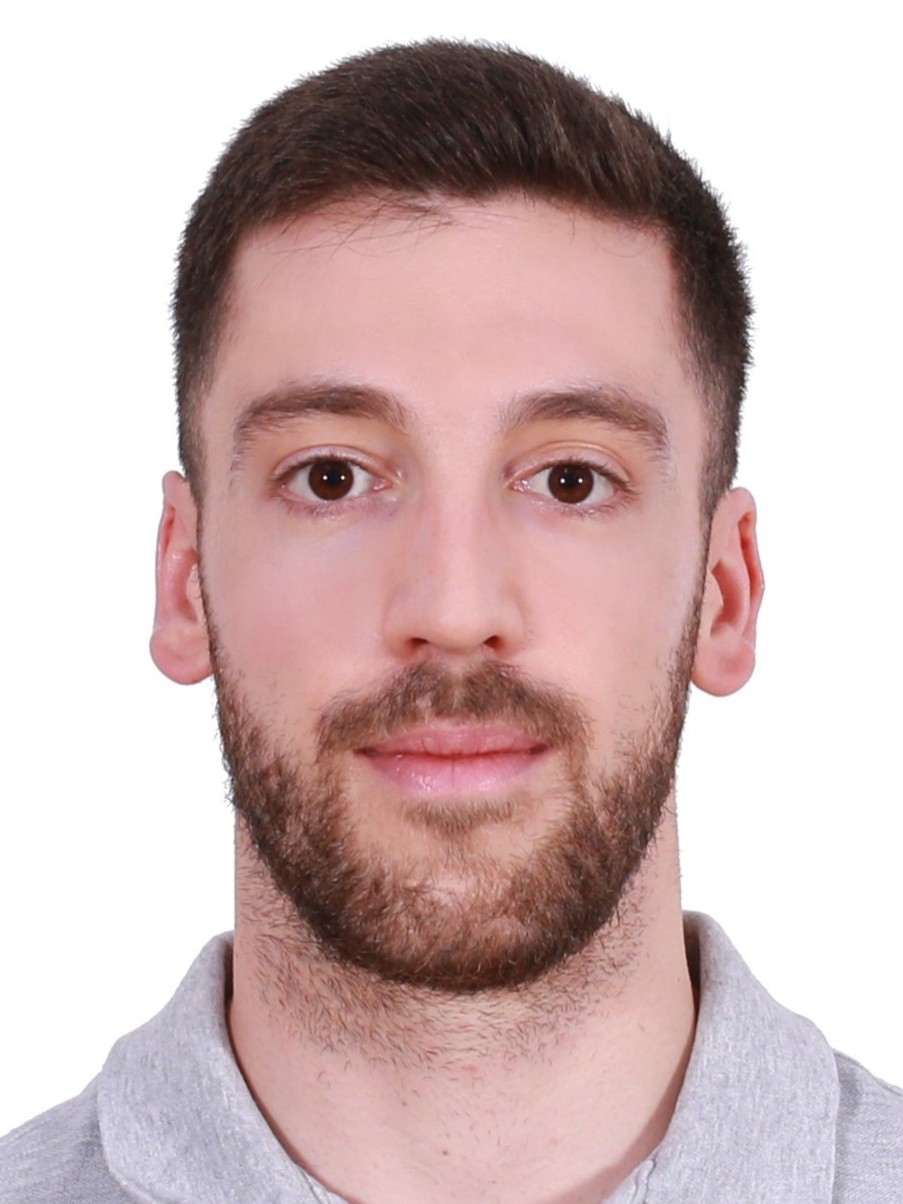}}]{Pedro Santos (Graduate Student Member, IEEE)}
	received his B.Sc. and M.Sc. degrees in Aerospace Engineering from Instituto Superior Técnico (IST), Lisbon, Portugal, in 2020, and 2022, respectively. He is currently pursuing a Ph.D. degree in Aerospace Engineering at IST in affiliation with the Associated Laboratory for Energy, Transports, and Aeronautics. His research interests are in the area of underactuated autonomous vehicles with a focus on nonlinear and robust Guidance, Navigation, and Control (GNC) algorithms for launch vehicles.
\end{IEEEbiography}

\begin{IEEEbiography}[{\includegraphics[width=1in,height=1.25in,clip,keepaspectratio]{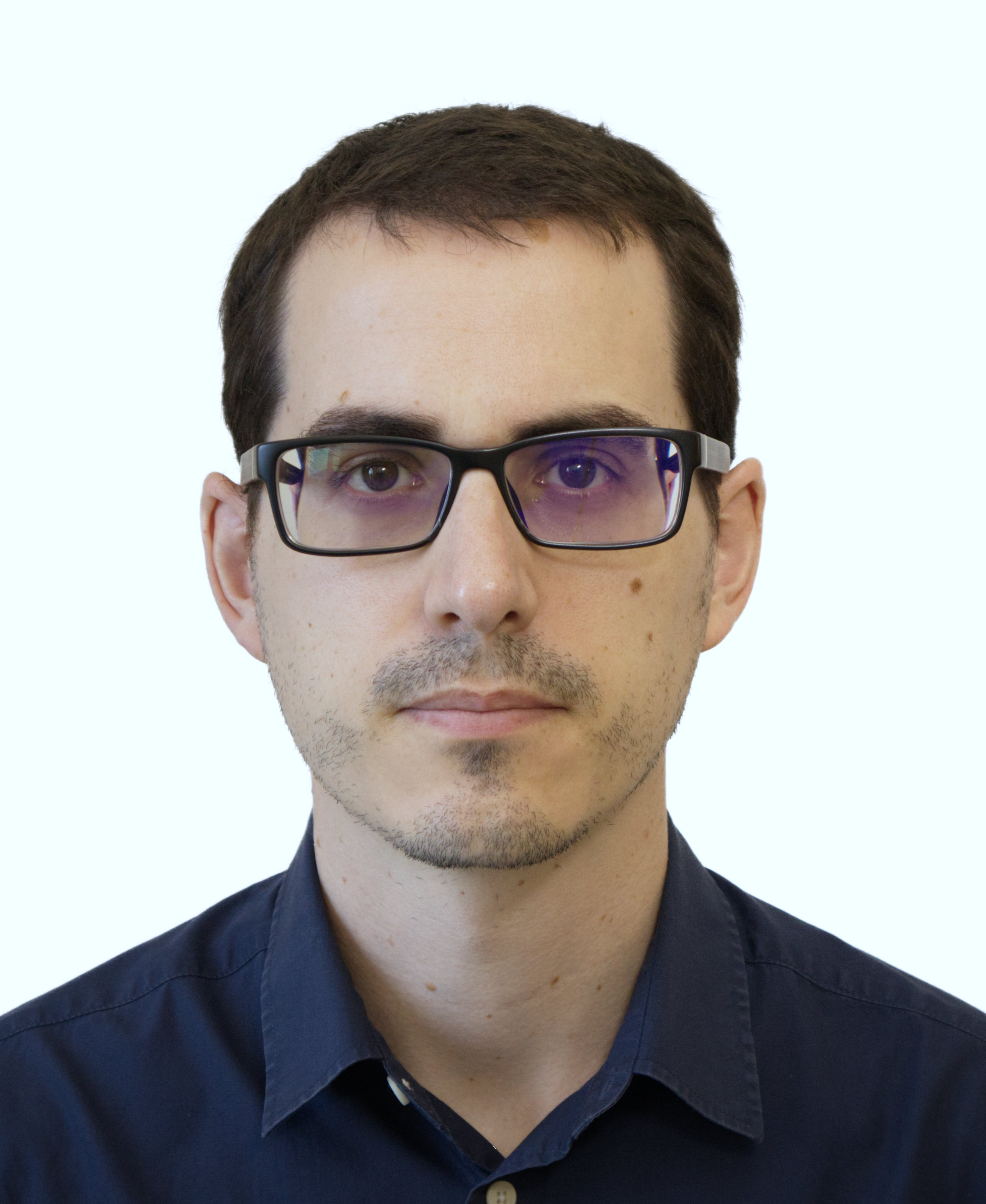}}]{Joel Reis}
	received the M.Sc. degree in Aerospace Engineering from the Instituto Superior Técnico, Lisbon, Portugal, in 2013, and the Ph.D. degree in Electrical and Computer Engineering from the University of Macau, Macau, in 2019. He is currently an Assistant Professor with the Faculty of Science and Technology, University of Macau. His research interests include estimation and control theory for autonomous vehicles.
\end{IEEEbiography}

\begin{IEEEbiography}[{\includegraphics[width=1in,height=1.25in,clip,keepaspectratio]{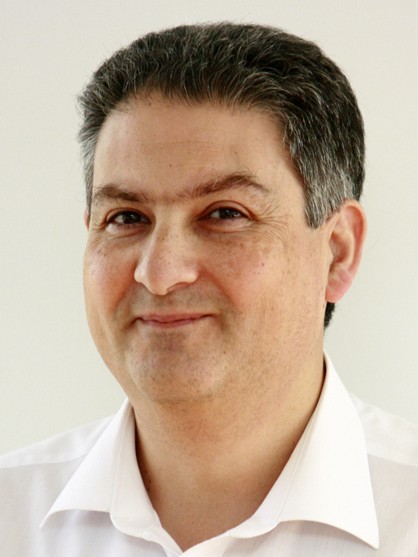}}]{Paulo Oliveira (Senior Member, IEEE)}
	received the Ph.D. degree in Electrical and Computer Engineering, in 2002, respectively, and the Habilitation in Mechanical Engineering in 2016, all from Instituto Superior Técnico (IST), Lisbon, Portugal. Since 2020, he holds a joint position as Full Professor in the Mechanical Engineering and Electrotechnical and Computer Engineering Departments of IST. He is the Coordinator of the Doctoral Program in Aerospace Engineering and member of the Scientific Council, both at IST. His research interests are on Autonomous Robotic Vehicles with a focus on Mechatronic Systems Integration, and Guidance, Navigation and Control Systems (GNC). He is author or coauthor of more than 100 journal and 200 conference papers, and participated in more than 40 European and Portuguese research projects, over the last 35 years. 
\end{IEEEbiography}

\begin{IEEEbiography}[{\includegraphics[width=1in,height=1.25in,clip,keepaspectratio]{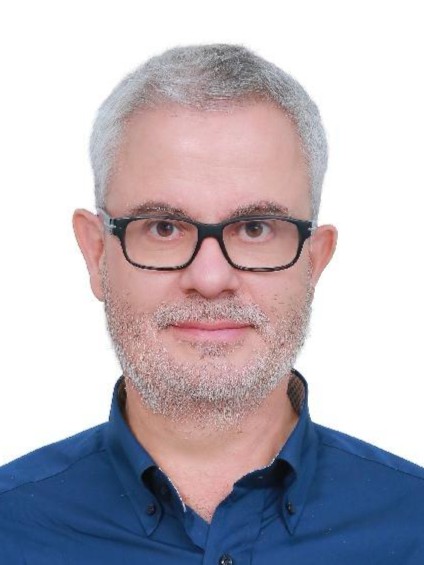}}]{Carlos Silvestre (Senior Member, IEEE)}
	received his Licenciatura and M.Sc. in Electrical and Computer Engineering, his Ph.D. in Control Science, and his Habilitation in Electrical and Computer Engineering from Instituto Superior Técnico (IST), Lisbon, Portugal, in 1987, 1991, 2000, and 2011, respectively. He has been with the Department of Electrical and Computer Engineering at IST since 2000, where he is a Professor in Systems, Decision, and Control, currently on leave since 2012. He is currently a Professor and Head of the Department of Electrical and Computer Engineering at the University of Macau’s Faculty of Science and Technology. His research interests include linear and nonlinear control, estimation theory, hybrid systems, multi-agent control, networked control, inertial navigation, and machine learning for autonomous systems, with a focus on unmanned ocean and aerial vehicles.
\end{IEEEbiography}

\end{document}